\documentclass[12pt]{article}
\usepackage{amsmath,amssymb}
\usepackage{libertine}
\usepackage{libertinust1math}
\usepackage{natbib}
\usepackage{dsfont}
\usepackage{xcolor}
\usepackage{tikz}
\usetikzlibrary{arrows.meta,positioning}
\usepackage{rotating}
\usepackage{tabularx}
\usepackage{array}
\usepackage{booktabs}
\usepackage{makecell}
\usepackage{refcount}
\usepackage[hidelinks]{hyperref}
\usepackage[shortlabels]{enumitem}
\usepackage{graphicx}
\usepackage{subcaption}
\usepackage{placeins}
\usepackage[margin=1in]{geometry}
\usepackage{amsthm}
\newtheorem{proposition}{Proposition}
\newtheorem{assumption}{Assumption}

\newtheorem{remark}{Remark}

\usepackage[nameinlink,noabbrev]{cleveref}
\usepackage[all,defaultlines=3]{nowidow}
\crefname{equation}{Equation}{Equations}
\Crefname{equation}{Equation}{Equations}
\crefname{algorithm}{Algorithm}{Algorithms}
\Crefname{algorithm}{Algorithm}{Algorithms}
\crefname{figure}{Figure}{Figures}
\Crefname{figure}{Figure}{Figures}
\crefname{table}{Table}{Tables}
\Crefname{table}{Table}{Tables}
\crefname{section}{Section}{Sections}
\Crefname{section}{Section}{Sections}
\crefname{subsection}{Section}{Sections}
\Crefname{subsection}{Section}{Sections}
\crefname{subsubsection}{Section}{Sections}
\Crefname{subsubsection}{Section}{Sections}
\crefname{assumption}{Assumption}{Assumptions}
\Crefname{assumption}{Assumption}{Assumptions}
\crefname{proposition}{Proposition}{Propositions}
\Crefname{proposition}{Proposition}{Propositions}
\crefname{remark}{Remark}{Remarks}
\Crefname{remark}{Remark}{Remarks}
\crefname{lemma}{Lemma}{Lemmas}
\Crefname{lemma}{Lemma}{Lemmas}
\crefname{theorem}{Theorem}{Theorems}
\Crefname{theorem}{Theorem}{Theorems}
\crefname{appendix}{Appendix}{Appendices}
\Crefname{appendix}{Appendix}{Appendices}
\newcommand{\proofsection}[3]{\section{Proof of \texorpdfstring{\Cref*{#1}}{Proposition \getrefnumber{#1}}: #2}\label{#3}}
\newcommand{\mcY}{\mathcal{Y}}
\newcommand{\mcG}{\mathcal{G}}
\newcommand{\R}{\mathbb{R}}
\newcommand{\E}{\mathbb{E}}
\newcommand{\Pp}{\mathbb{P}}
\newcommand{\indep}{\mathrel{\perp\!\!\!\perp}}
\DeclareMathOperator*{\argmax}{arg\,max}
\usepackage{algorithm}
\usepackage{algpseudocode}
\algrenewcommand\algorithmicrequire{\textbf{Input:}}
\algrenewcommand\algorithmicensure{\textbf{Output:}}

\title{Causal Inference with Unstructured Outcomes}

\author{
  Kevin Christian Wibisono\\
  Department of Statistics\\
  University of Michigan, Ann Arbor\\
  kwib@umich.edu
  \and
  Yixin Wang\\
  Department of Statistics\\
  University of Michigan, Ann Arbor\\
  yixinw@umich.edu
  }

\date{}

\begin{document}

\maketitle

\begin{abstract}
\noindent
Causal inference has traditionally centered on scalar outcomes: whether a patient recovers, how much a worker earns, or how many visits a website receives. Modern studies increasingly ask causal questions about outcomes with richer form, such as clinical notes, open-ended survey responses, and images. A hospital may want to know how an AI documentation tool changes the notes physicians write, or how a nurse training program alters what patients say in survey responses. For such outcomes, the usual average treatment effect is ill-defined: one cannot meaningfully subtract one text or image from another. To this end, we propose a causal query for unstructured outcomes. The key idea is to learn what features of the outcome are most causally affected by the treatment, which we call the \textit{maximally contrasting feature} (MCF). To estimate the MCF, we learn a feature-scoring function that maps each outcome to a scalar and exposes the sharpest contrast between treated and control potential outcomes. We develop identification conditions and estimation algorithms for this query, and extend it to heterogeneous effects by allowing the feature-scoring function to depend on observed covariates. We also handle settings where both the treatment and the outcome are unstructured. Empirical studies on text and images show that the algorithm recovers salient aspects of an outcome changed by a treatment.
\end{abstract}

\section{Introduction}

Causal inference has traditionally centered on scalar outcomes. A treatment is assigned, an outcome is measured, and the estimand compares mean potential outcomes under treatment and control. This setting underlies much of applied causal inference: whether a patient recovers, how long a patient survives, how much a worker earns, or how many visits a website receives.

Modern studies increasingly ask causal questions about richer unstructured objects. The outcome may be a clinical note, an open-ended survey response, a pathology image, or another unstructured record. A hospital may want to learn how an AI documentation tool changes physicians' notes. A public-health team may want to learn how a nurse-training program changes what patients say in survey responses. A medical-imaging study may want to learn how a reconstruction algorithm changes the artifacts, texture, or contrast of abdominal CT images. In these settings, the unstructured object is the evidence. Reducing it to a prespecified scalar before analysis can discard the changes that motivated the study.

We use clinical documentation as a running example, though the tools apply equally to text, images, and other unstructured outcomes. Suppose a hospital introduces an AI tool that helps clinicians draft notes. Let $A=1$ indicate that a note was written with the tool and $A=0$ that it was written without it. The raw outcome $Y$ is the resulting note. Treatment assignment is typically non-random: tool use may depend on visit type, clinician, patient characteristics, and time pressure, all of which may also shape the note. These variables are confounders, and a causal analysis must adjust for them. The scientific question is open-ended. Which aspects of the note change because the tool was used? Notes may become longer, more formal, more complete, more templated, more technical, or more uniform in structure. Their content may also change: clinicians may ask different questions, document different discussions, or engage with patients in different ways.

\textbf{Why the usual ATE breaks down.}
Let $Y(a)$ denote the raw potential outcome under treatment level $a$. In the documentation example, $Y(1)$ is the note that would be written with AI assistance, and $Y(0)$ is the note that would be written without it. For scalar outcomes, the average treatment effect
\[
    \mathbb E\{Y(1)-Y(0)\},
\]
has a direct interpretation. For documents and images, the expression has no analogous meaning. Clinical notes do not admit a canonical subtraction, images do not have a privileged direction of comparison, and an average document is rarely the object of scientific interest.

A common response is to represent and score the outcome. Let $g:\mathcal Y\to[0,1]$ be a bounded feature-scoring function. Then
\[
    \theta(g)
    =
    \mathbb E\!\left[g\{Y(1)\}-g\{Y(0)\}\right],
\]
is an ordinary average treatment effect for the scalar outcome $g(Y)$. The feature $g$ specifies the aspect of the object being studied. For a clinical note, $g(Y)$ might measure formality, brevity, completeness, clarity, technicality, or templating. For an image, it might measure blur, contrast, texture, sharpness, or another visual attribute.

This formulation shifts the original difficulty into a new question: which feature should be studied? The space of possible text or image features is enormous, and most features are irrelevant to a given intervention. One approach is to fix features in advance using a hand-coded codebook, topic model, embedding coordinate, or other dimension-reduction algorithm, and then estimate effects on those summaries \citep{egami2022make}. Such features are useful when they are themselves the scientific target. They are less useful when the target may lie outside a codebook, or may not be cleanly represented by any individual coordinate.

\textbf{The maximally contrasting feature.}
We make ``choosing which features to study'' part of the causal query. We consider a prespecified class $\mathcal G$ of bounded scalar functions $g:\mathcal Y\to[0,1]$ and search for the function that maximizes the feature-specific causal contrast $\theta(g)$.

We call the resulting function the \emph{maximally contrasting feature} (MCF). In the documentation example, the MCF is the represented feature whose mean potential value changes most when comparable notes are written with, rather than without, AI assistance. In an abdominal-imaging study, it may be the represented visual feature whose value changes most under a new reconstruction algorithm. The learned function~$g$ is a feature-scoring function: it maps each unstructured outcome to a scalar chosen to expose the sharpest causal contrast between treated and control potential outcomes. The MCF therefore defines a causal query for unstructured outcomes without requiring the analyst to know in advance whether the relevant change is lexical, stylistic, semantic, visual, or structural. We develop identification conditions and estimation algorithms for this query.

\Cref{fig:mcf-overview} shows this construction in the clinical-documentation example: adjust for factors that affect both tool use and the resulting notes, search over bounded note scores for the largest causal contrast, and interpret the learned score through recurring differences between low- and high-scoring notes.


\begin{figure*}[t]
\centering
\begin{minipage}[t]{0.22\textwidth}
\vspace{0pt}
\raggedright
\small
\textbf{(a) Causal question}

\vspace{3pt}
\centering
\begin{tikzpicture}[
  scale=0.88,
  transform shape,
  var/.style={
    circle,
    draw=black!55,
    line width=0.7pt,
    minimum size=9mm,
    text=white,
    font=\bfseries
  },
  causal/.style={
    -{Latex[length=2.2mm,width=1.7mm]},
    line width=1.1pt,
    draw=blue!65!black
  },
  direct/.style={
    -{Latex[length=2.2mm,width=1.7mm]},
    line width=1.1pt,
    draw=orange!85!black
  },
  label/.style={font=\tiny, align=center}
]
  \node[var, fill=blue!65!black] (x) at (1.50,2.05) {\(X\)};
  \node[var, fill=blue!65!black] (a) at (0.50,0.45) {\(A\)};
  \node[var, fill=green!55!black] (y) at (2.50,0.45) {\(Y\)};

  \draw[causal] (x) -- (a);
  \draw[causal] (x) -- (y);
  \draw[direct] (a) -- (y);

  \node[label, above=1mm of x] {visit, clinician,\\patient factors};
  \node[label, below=1mm of a] {AI documentation\\tool};
  \node[label, below=1mm of y] {clinical-note\\outcome};
\end{tikzpicture}

\vspace{4pt}
\begin{tikzpicture}
  \node[
    rounded corners=3pt,
    draw=orange!75!black,
    fill=orange!7,
    line width=0.7pt,
    align=center,
    font=\scriptsize\bfseries,
    text width=3.20cm,
    inner sep=5pt
  ] {Which aspects of the notes\\change because the AI tool is used?};
\end{tikzpicture}
\end{minipage}\hfill
\begin{minipage}[t]{0.44\textwidth}
\vspace{0pt}
\raggedright
\small
\textbf{(b) Learn the maximally contrasting features (MCF)}

\vspace{2pt}
{\scriptsize
Each bounded score \(g\in\mathcal G\) defines a causal contrast:
\[
  \theta(g)
  =
  \mathbb E\!\left[
    g\{Y(1)\}-g\{Y(0)\}
  \right],
  \qquad
  g^\star\in\arg\max_{g\in\mathcal G}\theta(g).
\]
}

\vspace{-4pt}
\centering
\begin{tikzpicture}[
  x=0.82cm,
  y=0.78cm,
  control/.style={
    circle,
    draw=blue!70!black,
    fill=blue!65!black,
    line width=0.5pt,
    minimum size=3.1mm,
    inner sep=0pt
  },
  treated/.style={
    circle,
    draw=orange!90!black,
    fill=orange!85!black,
    line width=0.5pt,
    minimum size=3.1mm,
    inner sep=0pt
  },
  gap/.style={draw=black!40, line width=1.0pt},
  bestgap/.style={draw=orange!85!black, line width=2.0pt},
  row/.style={font=\scriptsize, anchor=east},
  axislabel/.style={font=\tiny, text=black!75}
]
  \node[control] (legend0) at (2.00,4.25) {};
  \node[axislabel, anchor=west] at (2.18,4.25) {without AI tool};
  \node[treated] (legend1) at (5.20,4.25) {};
  \node[axislabel, anchor=west] at (5.38,4.25) {with AI tool};

  \node[row] at (1.30,3.35) {candidate \(g_1\)};
  \node[row] at (1.30,2.45) {candidate \(g_2\)};
  \node[row] at (1.30,1.55) {candidate \(g_3\)};
  \node[row, text=orange!85!black, font=\scriptsize\bfseries]
    at (1.30,0.65) {MCF \(g^\star\)};

  \draw[gap] (2.40,3.35) -- (3.12,3.35);
  \node[control] at (2.40,3.35) {};
  \node[treated] at (3.12,3.35) {};

  \draw[gap] (3.22,2.45) -- (3.58,2.45);
  \node[control] at (3.22,2.45) {};
  \node[treated] at (3.58,2.45) {};

  \draw[gap] (2.10,1.55) -- (3.72,1.55);
  \node[control] at (2.10,1.55) {};
  \node[treated] at (3.72,1.55) {};

  \draw[bestgap] (1.86,0.65) -- (5.12,0.65);
  \node[control] at (1.86,0.65) {};
  \node[treated] at (5.12,0.65) {};
  \node[
    font=\tiny\bfseries,
    text=orange!85!black,
    anchor=south
  ] at (3.49,0.78) {largest adjusted gap};

  \draw[-{Latex[length=1.6mm,width=1.2mm]}, black!55, line width=0.6pt]
    (1.65,0.05) -- (5.55,0.05);
  \foreach \x/\lab in {1.65/0,3.60/{0.5},5.55/1}
    {
      \draw[black!45] (\x,0.00) -- (\x,0.11);
      \node[axislabel, anchor=north] at (\x,-0.03) {\lab};
    }
  \node[axislabel, anchor=north] at (3.60,-0.40)
    {mean feature score};
\end{tikzpicture}

\vspace{3pt}
\raggedright
{\scriptsize
\textcolor{orange!85!black}{\textbf{Adjust for \(X\)}} so the plotted gaps
compare the same population under treatment and control.
}
\end{minipage}\hfill
\begin{minipage}[t]{0.30\textwidth}
\vspace{0pt}
\raggedright
\small
\textbf{(c) Interpret the MCF}

\vspace{2pt}
{\footnotesize
Score observed notes with \(g^\star\), then inspect recurring differences
across low- and high-scoring notes.}

\vspace{4pt}
\centering
\begin{tikzpicture}[
  low/.style={
    rounded corners=3pt,
    draw=blue!55!black,
    fill=blue!5,
    line width=0.7pt,
    align=left,
    font=\scriptsize,
    text width=3.45cm,
    inner sep=5pt
  },
  high/.style={
    rounded corners=3pt,
    draw=orange!80!black,
    fill=orange!8,
    line width=0.8pt,
    align=left,
    font=\scriptsize,
    text width=3.45cm,
    inner sep=5pt
  },
  insight/.style={
    rounded corners=5pt,
    fill=black!7,
    draw=black!30,
    align=center,
    font=\scriptsize\bfseries,
    text width=3.35cm,
    inner xsep=5pt,
    inner ysep=3pt
  }
]
  \node[low] (low) {\textbf{Low \(g^\star(Y)\)}\\[-1pt]
    \emph{``fever down. keep fluids.''}\\[-1pt]
    \emph{``call if cough gets worse.''}};
  \node[
    font=\scriptsize\bfseries,
    text=orange!85!black,
    text width=3.45cm,
    align=center,
    below=2mm of low
  ] (compare)
    {\(\boldsymbol{\downarrow}\)\; increasing feature score};
  \node[high, below=2mm of compare] (high) {\textbf{High \(g^\star(Y)\)}\\[-1pt]
    \emph{``Assessment: pain improved.\\
    Plan: continue medication.''}\\[-1pt]
    \emph{``Follow-up: blood pressure stable.\\
    Return in three months.''}};
  \node[insight, below=2.5mm of high] (insight)
    {recurring patterns give\\\(g^\star\) its meaning};
\end{tikzpicture}
\end{minipage}

\caption{\textbf{The MCF learns which feature of an unstructured outcome has
the largest causal contrast.} (a) In the clinical-documentation example,
pre-treatment covariates \(X\) affect both use of the AI documentation tool
\(A\) and the clinical-note outcome \(Y\). (b) Each bounded
\(g\in\mathcal G\) scores the notes and defines a treated-minus-control causal
contrast. After adjustment for \(X\), the maximally contrasting feature (MCF) \(g^\star\) is the feature with
the largest causal contrast. (c) Inspecting recurring differences between low- and
high-scoring notes gives the learned feature its substantive interpretation.}
\label{fig:mcf-overview}
\end{figure*}

Causal effects often vary across units. An AI documentation tool may change notes differently for complex visits than for routine visits, or for junior clinicians than for senior clinicians. A CT reconstruction algorithm may change image texture differently across scanner types or patient body habitus. We extend the MCF to heterogeneous effects by allowing the feature-scoring function to depend on observed covariates. The covariate-adaptive MCF asks which feature of the outcome is most affected for units with a given covariate profile, or which feature best captures treatment-effect variation across the population.

\textbf{Unstructured treatments and outcomes.}
The same difficulty arises when the treatment is itself unstructured. In text-to-image generation, the treatment may be a prompt and the outcome the generated image. A researcher may want to learn which attributes of a prompt causally improve which qualities of the image. In clinical care, the treatment may be an unstructured radiation plan, including dose distribution, beam geometry, and organ-at-risk constraints, while the outcome may be a follow-up CT scan. The causal question may ask which aspects of the plan improve which aspects of post-treatment image appearance. In writing, the treatment may be an AI suggestion and the outcome the final draft; the goal may be to learn which kinds of suggestions most improve specificity, voice, structure, or other dimensions of draft quality.

When treatment and outcome are both unstructured, we extend the MCF to learn a pair of functions: an influential treatment feature $f$ and a contrasting outcome feature $g$. The treatment feature identifies the aspect of the treatment being shifted, while the outcome feature identifies the aspect of the outcome that changes most under that supported shift.

\textbf{Contributions.} The paper makes five contributions:
\begin{enumerate}
\item We propose a causal query for causal inference with unstructured outcomes: the maximally contrasting feature (MCF), which learns a bounded feature-scoring function that captures the represented aspect of the outcome with the largest treatment-induced contrast.
\item We establish identification conditions for the MCF, develop inverse-propensity-weighted estimation algorithms, and study their consistency and first-order behavior.
\item We characterize the population behavior of the MCF and develop budgeted and covariate-adaptive versions of the MCF for interpretation.
\item We extend the MCF algorithm to settings in which both treatment and outcome are unstructured, via learning paired treatment and outcome feature-scoring functions with a covariate-adjusted objective estimated using matched negative-control outcomes.
\item We illustrate the MCF algorithm across empirical studies with unstructured outcomes, including text-based outcomes, image-based outcomes, and simultaneous text-based treatments and outcomes.
\end{enumerate}

\subsection{Related work}

The work draws on six themes of related work.

\textbf{Text-based causal adjustment.}
Text-based causal inference remains an emerging area, and much existing work uses text as a confounder, proxy, adjustment variable, or mediator rather than as the outcome itself. For example, \citet{keith2020text}, \citet{mozer2020matching}, and \citet{roberts2020adjusting} study how text can be used to adjust for confounding when estimating causal effects on scalar outcomes. \citet{wood2018challenges} study measurement problems that arise when text classifiers are used in causal analyses. Other work develops text representations and methods for causal adjustment or identification with latent and proxy confounding~\citep{veitch2020adapting,weld2022adjusting,chen2024proximal}. Work on multimodal adjustment and causal annotation also uses unstructured data to improve measurement or adjustment~\citep{klaassen2024doublemldeep,nwankwo2025batch}. \citet{zhou2026integrating} use unstructured product descriptions as inputs to outcome and propensity models for observational treatment-effect estimation, while \citet{arbour2026variance} use AI predictions from rich unstructured inputs as auxiliary covariates for variance reduction in randomized experiments. \citet{jin2026partial} treat generated output as a mediator between model choice and a scalar outcome; they combine a small randomized experiment with an offline simulator to identify causal model values and use confounded observational logs only to improve estimation. For a comprehensive overview of causal inference in NLP, see \citet{feder2022causal}. Our setting is different: the unstructured object is itself the outcome whose causal change we seek to characterize.

\textbf{Text outcomes and codebooks.}
The closest work studies causal inference with text-based outcomes. \citet{egami2022make} propose using a \textit{codebook} function to derive a low-dimensional text representation and estimate causal effects in studies with text-based treatments or outcomes, including randomized experiments. They emphasize sample splitting so that the codebook is constructed separately from the sample used for causal estimation, which helps satisfy the consistency assumption. Their workflow fixes the codebook before causal estimation. In contrast, we focus on observational studies and select a causally relevant feature-scoring function by optimizing the treatment-induced contrast.

\textbf{Empirical causal analyses of text outcomes.}
Several empirical studies estimate causal effects on pre-specified or model-derived text summaries. \citet{gill2015judicial} investigate how the random assignment of a female or non-white judge changes the use of legal terminology in final rulings. \citet{sridhar2019estimating} use latent Dirichlet allocation (LDA) topic models to assess how reply tones in debates affect linguistic and sentiment shifts in subsequent responses, while controlling for users' ideologies and emotions. More recently, \citet{modarressi2025causal} develop a method for text outcomes in randomized experiments and learn causal themes by prompting an LLM to obtain flexible document descriptions that capture systematic differences between treatment and control groups. These works share our goal of treating text as a causal outcome, but they rely on pre-specified summaries, topic models, hand-designed or generative codebooks, or LLM-prompted descriptions. Our approach instead learns the outcome feature from the causal contrast itself and accommodates observational confounding.

\textbf{High-dimensional outcomes.}
Related work studies causal inference with high-dimensional outcomes. \citet{jeong2024identifying} study sparse treatment effects across many measured outcomes and develop algorithms for identifying a sparse set of affected outcomes in randomized experiments. This is conceptually related to our goal of finding a low-dimensional summary of treatment-induced change. However, our focus is on unstructured outcomes represented through text, image, or other embedding functions, and on learning an interpretable causally relevant feature rather than selecting among a fixed set of measured outcomes.

\textbf{Unstructured treatments.}
A complementary literature treats language or other unstructured objects as treatments rather than outcomes. This includes work on discovering treatments from text corpora~\citep{fong2016discovery}, estimating causal effects of linguistic properties~\citep{pryzant2021causal}, transporting language effects across text distributions~\citep{lin2023texttransport}, optimizing language models for human preferences using causal objectives~\citep{lin2024cpo}, and estimating isolated effects of natural-language attributes~\citep{lin2025isolated}. Recent work also uses generative AI to define, manipulate, or estimate effects of unstructured text treatments~\citep{imai2025gpi,lin2025ivrepr,shi2025semantic,imai2024causal,nakamura2026dynamic}. These papers study how to define or estimate effects of unstructured treatments. By contrast, our main contribution is outcome-side: we define and estimate the causally relevant feature of an unstructured outcome.

\textbf{Representations of unstructured outcomes.}
Finally, our algorithm depends on a representation of the unstructured outcome. Such representations may come from hand coding, language or multimodal embeddings, or style-specific representations \citep{reimers2019sentence,wegmann2022same,duquenne2023sonar,patel2023learning,patel2024styledistance}. A broad and rapidly developing literature on causal representation learning studies when latent causal variables can be recovered from low-level observations; see \citet{scholkopf2021towards,moran2026towards} for reviews. This literature is complementary: it concerns how representations are obtained or identified, whereas we define and estimate causal features after a representation has been chosen.

\subsection{Organization of the paper}
\Cref{sec:mcf-methods} defines the MCF for represented unstructured outcomes, states its identification assumptions, characterizes its population meaning, and develops an estimator. \Cref{sec:extension-hd-treatment-outcome} extends the construction to settings in which both the treatment and the outcome are unstructured. \Cref{sec:experiments} evaluates the MCF algorithm in studies with text-based outcomes, image-based outcomes, and text-based treatments and outcomes. \Cref{sec:discussion} closes with a discussion.

\section{Causal inference with unstructured outcomes}
\label{sec:mcf-methods}

This section formulates a causal query for unstructured outcomes, gives identification conditions and estimation algorithms for this query, and extends the construction to heterogeneous effects. In the running example, the outcome is a clinical note. In other applications, it may be an open-ended survey response, an image, a video, or another unstructured object.

\subsection{Setup, potential outcomes, and the causal query}
\label{sec:mcf-setup-query}

\textbf{Observed data and potential outcomes.}
We observe i.i.d. data
\[
    Z_i=(X_i,A_i,Y_i),\qquad i=1,\ldots,N,
\]
where \(N\) is the sample size, \(X_i\) contains pre-treatment covariates, \(A_i\in\{0,1\}\) is the treatment, and \(Y_i\in\mcY\) is an unstructured outcome. For each treatment level \(a\in\{0,1\}\), let \(Y(a)\) denote the potential outcome that would be observed under treatment level \(a\). The analysis treats \(Y\) as the outcome. In the running example, \(A=1\) means that the clinical note was written with the documentation tool, \(A=0\) means that it was written without the tool, and \(Y\) is the note whose causal change is under study. Thus, \(Y(a)\) is the note that would be observed under documentation condition \(a\).

\textbf{Numerical representations of unstructured outcomes.}
Modern unstructured objects often come with numerical representations. A language or vision model can map a document or image into a vector from which much of the object can be retrieved, reconstructed, or generated. This map need not be a mathematical bijection: an arbitrary point in an embedding space may fail to correspond to any coherent note or natural image. For scientific analysis, however, the representation often supplies a useful working scale. Observed objects can be embedded, and neighborhoods in the representation space can be inspected by retrieving, reconstructing, or generating corresponding objects. Thus, \(Y\) is often treated as the outcome for the object that the study intends to measure.

Specifically, let \(O(a)\) be the raw potential outcome under treatment level~\(a\), and let \(Y(a)=\psi\{O(a)\}\) be its numerical representation. The raw outcome~\(O(a)\) may be a clinical note, an open-ended survey response, an image, or another unstructured object. The representation function~\(\psi\) may be a text embedding, an image embedding, a set of hand-coded variables, or a representation designed to isolate a modifiable component such as writing style or image texture. In the clinical-documentation example, \(O(a)\) is the note that would be written under documentation condition~\(a\), and \(Y(a)\) is its sentence embedding.

\textbf{Choosing the information in \(Y\).}
The choice of \(\psi\) determines which information from the raw object is available for causal analysis, and different choices of \(\psi\) support different causal questions. In clinical documentation, diagnoses, medications, test results, and patient history are largely fixed by the encounter. A documentation tool is more likely to change style: tone, formality, verbosity, organization, and technical language. A representation that mixes content and style may capture differences in patient mix or visit type rather than the modifiable effect of the tool. The analyst may therefore use a general-purpose embedding, such as SentenceTransformer~\citep{reimers2019sentence} or SONAR~\citep{duquenne2023sonar}; a style representation that separates semantic content from tone or formality~\citep{wegmann2022same,patel2023learning,patel2024styledistance}; or a representation constructed with a language model, for example by generating or extracting descriptions that preserve content while varying style~\citep{wibisono2026causal}. The same principle applies to images: the representation should emphasize the visual component that the intervention can plausibly affect, such as contrast, texture, sharpness, anatomical deformation, or artifact pattern.

\textbf{Why a feature-scoring function is still needed.}
A numerical representation makes an unstructured outcome available to statistical analysis, but it does not by itself define the causal feature of interest. A text or image embedding gives coordinates for the object, but its axes rarely form scientific scales. One coordinate is usually a component of a compressed code, rather than an interpretable attribute such as formality, specificity, blur, or contrast. A raw difference between two embedding vectors is algebraically defined, but it does not reveal which aspect of the note, response, or image changed.

\textbf{The causal query.}
First suppose that the feature of interest is known. In the clinical-note example, the analyst may ask whether the AI tool makes notes more formal. Such a feature can often be represented by a feature-scoring function \(g:\mcY\to\R\), where larger values of \(g(Y)\) correspond to more formal notes. Other choices of \(g\) could score verbosity, organization, technical language, or any other measurable feature of the represented outcome. Once \(g\) is fixed, the causal question is ordinary: how much would the average score change if the same population of encounters were documented with the tool rather than without it? This gives the feature-specific causal contrast
\begin{equation}
\label{eq:mcf-theta}
    \theta(g)=\E\left[g\{Y(1)\}-g\{Y(0)\}\right].
\end{equation}
For fixed \(g\), this is the usual average treatment effect with \(g(Y)\) as the scalar outcome.

For unstructured outcomes, however, the relevant feature is often unknown in advance. A hospital may know that a documentation tool changes clinical notes, while remaining uncertain about whether the change is mainly formality, verbosity, organization, or some more subtle attribute of the text.

We propose to make choosing ``which features to study'' part of the causal query. Let \(\mcG\) be a prespecified class of feature-scoring functions \(g:\mcY\to\R\). The maximally contrasting feature is
\begin{equation}
\label{eq:mcf-definition}
    g^\star \in \argmax_{g\in\mcG} \theta(g).
\end{equation}
Equivalently, \(g^\star\) is the feature in \(\mcG\) whose average potential value changes most when treatment is set from \(0\) to \(1\). In the running example, \(g^\star(Y)\) is the learned note score that most strongly separates notes written with the documentation tool from the notes that would have been written without it.

This definition separates three choices. The representation \(\psi\) determines which information from the raw object is available for analysis. The function class \(\mcG\) determines which scalar summaries may be learned. The contrast \(\theta(g)\) selects, within that class, the feature most affected by treatment. Together, these choices define a causal query for unstructured outcomes without requiring the analyst to know in advance which aspect of the outcome the treatment will move.

\subsection{Causal identification of the MCF}
\label{sec:challenge-asm}

Two issues precede identification. The first is scale. The criterion in \Cref{eq:mcf-definition} is linear in the feature \(g\). Thus, if a feature has positive contrast, rescaling it by a large constant increases the objective without changing the ordering of outcomes. We fix this scale by restricting the feature class to bounded functions, typically \(0\leq g\leq1\). This normalization gives the learned feature a simple interpretation: values near one mark outcomes that are more characteristic of the treated potential-outcome distribution, and values near zero mark outcomes that are more characteristic of the control potential-outcome distribution.

The second issue is computational. Optimizing over all measurable functions is not feasible in finite samples, so we often set \(\mcG\) as a tractable bounded class, such as a fixed neural-network architecture with sigmoid output. A smooth linear score,
\[
    g(y)=\sigma(w^\top y),
\]
is a special case, where~\(w\) is the coefficient vector and~\(\sigma\) is the sigmoid function. The bounded output fixes the scale, while the fixed architecture makes the search estimable. For interpretation, it is useful to compare this practical class with the oracle class
\[
    \mcG_{\mathrm{oracle}}=\{g:0\leq g(y)\leq1\},
\]
which describes the population feature that would be selected without approximation error.

We now turn to identification. The assumptions are stated for the outcome \(Y\), before any feature is learned. This is because the feature search ranges over functions of \(Y\); the causal conditions must therefore support every candidate contrast considered by the procedure.

\begin{assumption}[SUTVA]
\label{asm:mcf-sutva}
Each unit has well-defined potential outcomes \(Y(0)\) and \(Y(1)\). One unit's treatment does not affect another unit's outcome. If \(A=a\), then \(Y=Y(a)\).
\end{assumption}

\begin{assumption}[Overlap]
\label{asm:mcf-overlap}
Let \(e_0(x)=\Pp(A=1\mid X=x)\) be the propensity score, and let \(P_X\) be the marginal covariate law. For \(P_X\)-almost every \(x\),
\[
    0<e_0(x)<1.
\]
For estimation results, we use the stronger condition that \(\eta<e_0(x)<1-\eta\) for an overlap margin~\(\eta>0\).
\end{assumption}

\begin{assumption}[Unconfoundedness]
\label{asm:mcf-unconfoundedness}
\[
    \{Y(0),Y(1)\}\indep A\mid X.
\]
\end{assumption}

\Cref{asm:mcf-sutva} links the observed outcome to the corresponding potential outcome. \Cref{asm:mcf-overlap} ensures that both treatment levels remain possible within the relevant covariate strata. \Cref{asm:mcf-unconfoundedness} says that the observed covariates block confounding between treatment and the outcome. In the clinical-note example, after adjusting for visit type, clinician behavior, patient complexity, and other recorded drivers of tool use, the remaining variation in AI tool use carries no additional information about the notes that would have appeared under either treatment level.

For any fixed measurable feature \(g\), define
\[
    \theta(g)=\E\{g(Y(1))-g(Y(0))\},
\]
and let
\[
    \mu_a^g(x)=\E\{g(Y)\mid A=a,X=x\}.
\]

\begin{proposition}[Identification of feature-specific contrasts]
\label{prop:mcf-identification}
Under \Cref{asm:mcf-sutva,asm:mcf-overlap,asm:mcf-unconfoundedness}, every fixed measurable feature \(g\) with finite expectation satisfies
\begin{equation}
\label{eq:mcf-adjustment}
    \theta(g)
    =
    \E\left[
        \E\{g(Y)\mid A=1,X\}
        -
        \E\{g(Y)\mid A=0,X\}
    \right]
    =
    \E\{\mu_1^g(X)-\mu_0^g(X)\}.
\end{equation}
Equivalently,
\begin{equation}
\label{eq:mcf-ipw-identification}
    \theta(g)
    =
    \E\left[
        \frac{A g(Y)}{e_0(X)}
        -
        \frac{(1-A)g(Y)}{1-e_0(X)}
    \right].
\end{equation}
\end{proposition}
The proof is in \Cref{app:pf-mcf-identification}. The proposition identifies the contrast for each fixed \(g\). Since the same observed-data formula applies throughout \(\mcG\), it identifies the population criterion \(g\mapsto\theta(g)\). Consequently, whenever this criterion has a unique maximizer in the chosen class, the selected feature is also identified. The learned feature is causal in this precise sense: it is selected by optimizing an identified causal contrast.

The next proposition records this conclusion and characterizes the oracle maximizer. Boundedness makes the optimization well-scaled; uniqueness then depends on the absence of ties in the population ordering.

\begin{proposition}[Identification and uniqueness of the MCF]
\label{prop:mcf-identification-uniqueness}
Suppose \Cref{asm:mcf-sutva,asm:mcf-overlap,asm:mcf-unconfoundedness} hold. Then \(\theta(g)\) is identifiable for every \(g\in\mcG\) by \Cref{eq:mcf-adjustment}. If \(\theta(g)\) has a unique maximizer in \(\mcG\), up to equality under the reference distribution of \(Y\), then \(g^*\) is identifiable.

In the oracle class \(\mcG_{\mathrm{oracle}}=\{g:0\le g\le1\}\), suppose the laws of \(Y(1)\) and \(Y(0)\) have densities \(p_1\) and \(p_0\) with respect to a common measure. Write
\[
    \Delta(y)=p_1(y)-p_0(y).
\]
Then every oracle maximizer satisfies
\begin{equation}
\label{eq:oracle-mcf}
    g^*(y)=\mathbf{1}\{\Delta(y)>0\}
    \quad\text{where } \Delta(y)\ne0.
\end{equation}
If the tie set \(\{y:\Delta(y)=0\}\) has probability zero, the oracle MCF is unique almost surely.
\end{proposition}
The proof is in \Cref{app:pf-mcf-identification-uniqueness}. The proposition gives the population interpretation of the MCF. It marks the region of the represented outcome space that treatment makes more likely. For clinical notes, it selects the kinds of represented notes that appear more often when physicians use the documentation tool than they would under control, after adjustment for \(X\).

A simple special case makes the interpretation concrete. Suppose an interpretable scalar summary \(H=h(Y)\) is the only represented feature whose distribution changes under treatment, in the sense that the density ratio \(p_1(y)/p_0(y)\) is a strictly increasing function of \(h(y)\). Then \Cref{eq:oracle-mcf} is equivalent to thresholding \(H\). In the running example, if more formal notes are increasingly more likely under treatment than under control, and no other aspect of the notes changes, then the oracle MCF is equivalent to thresholding formality. When treatment changes several correlated aspects of the representation, the MCF learns the scalar score that best separates the treated and control outcome distributions.

\subsection{Heterogeneous MCF}
\label{sec:hetero-te}

A global feature \(g(Y)\) summarizes the average distributional change over the study population. Treatment, however, may alter different aspects of an unstructured outcome in different subpopulations. A documentation tool might make routine-visit notes more templated while making complex-visit notes more complete. Let~\(\mathcal X\) denote the covariate space. To represent such patterns, we allow the maximally contrasting feature (MCF) to depend on the covariates,
\[
    g:\mathcal{Y}\times\mathcal{X}\to[0,1].
\]
For any such feature, define the covariate-adaptive contrast
\begin{equation}
\label{eq:mcf-theta-x}
    \theta_X(g)
    =
    \mathbb{E}\left[g\{Y(1),X\}-g\{Y(0),X\}\right].
\end{equation}
The heterogeneous MCF is then
\begin{equation}
\label{eq:heterogeneous-mcf-definition}
    g_X^*
    \in
    \arg\max_{g\in\mathcal{G}_X}\theta_X(g),
\end{equation}
where \(\mathcal{G}_X\) is a bounded class of feature-scoring functions from \(\mathcal{Y}\times\mathcal{X}\) to \([0,1]\). Thus \(g_X^*(y,x)\) describes which represented outcomes are most characteristic of treatment within the covariate profile \(x\).

The identification arguments above carry over after replacing \(g(Y)\) by \(g(Y,X)\). In particular, with
\[
    \mu_a^g(x)=\mathbb{E}\{g(Y,X)\mid A=a,X=x\},
\]
we have
\[
    \theta_X(g)=\mathbb{E}\{\mu_1^g(X)-\mu_0^g(X)\}.
\]
The resulting feature describes which represented aspects of the outcome change within subpopulations, rather than only on average over the population.

\begin{proposition}[Characterization of the heterogeneous oracle MCF]
\label{prop:ext-2}
Let
\[
    T(g)
    =
    \mathbb{E}\left[g\{Y(1),X\}-g\{Y(0),X\}\right]
    =
    \int g(y,x)\,p(x)\,\Delta(y\mid x)\,dy\,dx,
\]
where \(p(x)\) is the marginal covariate density. Define the conditional density difference by
\[
    \Delta(y\mid x)=p_1(y\mid x)-p_0(y\mid x),
\]
where $p_a(y\mid x)$ is the conditional density of \(Y(a)\) given $x$.
Here, \(\Delta(y\mid x)\) denotes the conditional difference between the treated and control potential-outcome densities. Over the oracle class \(0\le g\le1\) on \(\mathcal{Y}\times\mathcal{X}\), any maximizer satisfies
\[
    g^*(y,x)=\mathbf{1}\{\Delta(y\mid x)>0\}
    \quad\text{where }\Delta(y\mid x)\ne0.
\]
If the tie set \(\{(y,x):\Delta(y\mid x)=0\}\) has reference probability zero, this maximizer is unique almost surely.
\end{proposition}
The proof is in \Cref{app:pf-ext-2}. This characterization preserves the pointwise oracle interpretation while allowing the feature-scoring function to vary across covariate strata.

\subsection{Causal estimation of the MCF}
\label{sec:mcf-obs-data}

The preceding subsection defines the MCF as a population feature: among the feature-scoring functions in
\(\mcG\), it selects the one whose average potential value changes most under treatment. We now
estimate this feature from observational data. The presentation focuses on the homogeneous score
\(g(Y)\). The heterogeneous version uses the same construction after replacing \(g(Y;\beta)\) by
\(g(Y,X;\beta)\).

\textbf{A tractable search class.}
The population definition allows any prespecified class \(\mcG\). In a finite sample, we work with a
parameterized subfamily so that the search can be carried out by numerical optimization. Let
\[
    \mcG_\Theta=\{g(\cdot;\beta):\beta\in\Theta\},
\]
where \(\beta\in\Theta\subseteq\mathbb{R}^{n_\beta}\) is the parameter vector, \(n_\beta\) is its dimension, and \(\beta\) may index the weights of a neural network. We keep the scale constraint from the
population definition, for example by using a bounded output layer so that
\(0\le g(Y;\beta)\le 1\). The bound fixes the scale of the contrast; the parameterization turns the
search over features into a training problem.

The population target within this fitted class is
\begin{equation}
\label{eq:beta-target}
    \beta_0\in
    \argmax_{\beta\in\Theta}
    \theta(\beta),
    \qquad
    \theta(\beta)
    :=
    \E\left[g\{Y(1);\beta\}-g\{Y(0);\beta\}\right].
\end{equation}
Thus \(g(Y;\beta_0)\) is the bounded score in the chosen class whose average potential value changes
most when treatment is set from \(0\) to \(1\). In the clinical-note example, it is the learned note
score most affected by use of the documentation tool.

\textbf{The inverse-propensity-weighted (IPW) objective.} We propose to estimate the MCF using an inverse-propensity-weighted objective.
Let
\[
    e_0(X)=\Pp(A=1\mid X),
\]
be the propensity score. Under SUTVA, overlap, and unconfoundedness, the contrast in
\Cref{eq:beta-target} can be written in terms of observed outcomes
\begin{equation}
\label{eq:ipw-pop}
    \theta(\beta)
    =
    \E\!\left[
        \frac{A\,g(Y;\beta)}{e_0(X)}
        -
        \frac{(1-A)\,g(Y;\beta)}{1-e_0(X)}
    \right].
\end{equation}
This identity directly gives a learning objective. Once the propensity score is learned,
the remaining task is to score the observed notes, images, or responses and optimize over
\(\beta\).

Given an estimator \(\hat e\), define
\begin{equation}
\label{eq:ipw-obj}
    \hat\theta_N(\beta)
    =
    \frac{1}{N}
    \sum_{i=1}^N
    \left\{
        \frac{A_i\,g(Y_i;\beta)}{\hat e(X_i)}
        -
        \frac{(1-A_i)\,g(Y_i;\beta)}{1-\hat e(X_i)}
    \right\},
    \qquad
    \hat\beta\in\argmax_{\beta\in\Theta}\hat\theta_N(\beta).
\end{equation}
This separates the two tasks. The propensity model is learned once, using \(X\) and \(A\). The feature-scoring function is then learned by optimizing a weighted empirical contrast over the represented outcomes. In implementation, we use sample splitting or cross-fitting: estimate \(\hat e\) on one fold, optimize \Cref{eq:ipw-obj} on another, and rotate the folds.

A reader might ask why we do not instead estimate conditional means such as
\(\E\{g(Y;\beta)\mid X,A=a\}\) and plug them into the identified contrast. The difficulty is that
the regression target moves with \(\beta\). Each update of the feature-scoring function changes the outcome that
the regression must predict, creating a nested optimization problem for text and image embeddings.
The IPW objective avoids this loop because the propensity model \(e_0(X)\) is learned from the
treatment assignment mechanism and remains fixed as \(g\) is trained. We detail this reasoning in \Cref{sec:why-ipw}.

\textbf{Consistency.}
The next proposition states that the fitted parameter~$\hat{\beta}$ consistently estimates the population parameter~$\beta_0$ in \Cref{eq:beta-target} under standard regularity conditions.

\begin{proposition}[Consistency]
\label{prop:cons}
Let \(\Theta\) be compact and contain \(\beta_0\). Suppose \(\theta(\beta)\) has a unique maximizer
\(\beta_0\), the class
\[
    \{g(\cdot;\beta):\beta\in\Theta\},
\]
is Glivenko--Cantelli under the weighted law, overlap holds, and
\[
    \sup_x|\hat e(x)-e_0(x)|\xrightarrow{p}0,
\]
then any maximizer \(\hat\beta\) of \Cref{eq:ipw-obj} satisfies
\[
    \hat\beta\xrightarrow{p}\beta_0.
\]
\end{proposition}

The proof is in \Cref{app:pf-cons}. The compactness condition can be relaxed, as the following remark records.
\begin{remark}
\label{rmk:tight}
    The compactness assumption on $\Theta$ can be replaced by the \textit{tightness} assumption: For every $\epsilon > 0$, there is a compact subset $K_\epsilon \subset \Theta$ containing $\beta_0$ such that for all large enough $N$, we have $\mathbb{P}(\hat{\beta}_N \in K_\epsilon) \geq 1 -\epsilon$.
\end{remark}

\textbf{First-order behavior.}
The next result records the large-sample behavior of the learned MCF. Specifically, if the fixed-\(\beta\) contrast estimator has a regular first-order expansion, then the argmax estimator inherits that expansion through the first-order condition. Thus, when the true propensity score~$e_0$ is estimated well, the fitted parameter~$\hat{\beta}$ is not only consistent, but also asymptotically normal and semiparametrically efficient.

Let the score gradient be~$\ell(Y;\beta)=\partial g(Y;\beta)/\partial\beta$, let the observed-data vector be~$Z=(X,A,Y)$, and let~$e$ denote a candidate propensity score. Write~$\mathbb{P}$ and~$\mathbb{P}_N$ for the population and empirical expectation operators, respectively, and define the IPW moment function
\begin{equation}
\label{eq:ipw-moment}
\phi(Z;\beta,e)=\frac{A\ell(Y;\beta)}{e(X)}-\frac{(1-A)\ell(Y;\beta)}{1-e(X)}.
\end{equation}
The population pair~$(\beta_0,e_0)$ satisfies the moment condition~$\mathbb{P}\{\phi(Z;\beta_0,e_0)\}=0$, while the fitted pair~$(\hat{\beta},\hat e)$ satisfies~$\mathbb{P}_N\{\phi(Z;\hat{\beta},\hat e)\}=0$.

For every candidate parameter~$\beta$, let the population moment be~$\tau_\beta=\mathbb{P}\{\phi(Z;\beta,e_0)\}$ and define its IPW estimator as
\begin{equation}
\label{eq:tau-estimator}
\hat{\tau}_\beta=\frac{1}{N}\sum_{i=1}^N\left\{\frac{A_i\ell(Y_i;\beta)}{\hat e(X_i)}-\frac{(1-A_i)\ell(Y_i;\beta)}{1-\hat e(X_i)}\right\}.
\end{equation}
The observed outcome is indexed by~$i$ in both summands. Suppose this estimator is semiparametrically efficient for every fixed candidate parameter~$\beta$; that is,
\begin{equation}
\label{eq:tau-expansion}
\hat{\tau}_\beta-\tau_\beta=(\mathbb{P}_N-\mathbb{P})\{\gamma(Z;\beta,e_0)\}+o_p(N^{-1/2}),
\end{equation}
where the efficient influence representation is
\begin{equation}
\label{eq:efficient-representation}
\gamma(Z;\beta,e)=\frac{A\{\ell(Y;\beta)-\mu_{1\beta}(X)\}}{e(X)}-\frac{(1-A)\{\ell(Y;\beta)-\mu_{0\beta}(X)\}}{1-e(X)}+\mu_{1\beta}(X)-\mu_{0\beta}(X),
\end{equation}
and the outcome-regression gradient is~$\mu_{a\beta}(X)=\mathbb{E}\{\ell(Y;\beta)\mid X,A=a\}$. Finally, let the derivative matrix be~$V(\beta,e)=\partial\mathbb{P}\{\phi(Z;\beta,e)\}/\partial\beta$. Under the following regularity conditions, the fitted parameter~$\hat{\beta}$ is asymptotically normal and semiparametrically efficient.

\begin{proposition}[Asymptotic normality and efficiency]
\label{prop:eff}
Suppose that all the following conditions are met:
\begin{enumerate}[label=(\alph*)]
    \item The conditions of \Cref{prop:cons} hold, so that \(\hat\beta\xrightarrow{p}\beta_0\);
    \item The propensity estimator~$\hat e$ is trained on observations independent of those used in the estimating equation for~$\hat\beta$, or the analogous independence is obtained by cross-fitting;
    \item The IPW-based estimator $\hat{\tau}_\beta$ for $\tau_\beta$ is semiparametrically efficient for any fixed $\beta$;
    \item The function class \(\{\phi(z;\beta,e):\beta\in\mathbb{R}^{n_\beta}\}\) is Donsker in \(\beta\) for any fixed~\(e\);
    \item The moment map $\beta \rightarrow \mathbb{P}\{\phi(Z;{\beta}, {e})\}$ is differentiable at \(\beta_0\), uniformly over candidate propensity scores~\(e\). The derivative matrix~\(V(\beta_0,e)\) is non-singular, and \(V(\beta_0,\hat e)\xrightarrow{p}V(\beta_0,e_0)\).
\end{enumerate}
Then, $\hat{\beta}$ is an asymptotically normal and semiparametrically efficient estimator for $\beta_0$.
\end{proposition}
The proof is in \Cref{app:pf-eff}. The following remark clarifies the efficiency assumption used by this result.
\begin{remark}
\Cref{prop:eff} relies on the assumption that the IPW-based estimator is semiparametrically efficient; it concerns the fixed-\(\beta\) contrast estimator. Although classical IPW estimators are not generally efficient, semiparametric efficiency can be attained when the propensity score is estimated flexibly, such as nonparametrically, under suitable regularity conditions, as established in prior work, e.g.,  \citet{hirano2003efficient}, \citet{ertefaie2023nonparametric}, and \citet{van2022efficient}. The same argmax expansion also applies if \Cref{eq:ipw-obj} is replaced by an efficient augmented contrast.
\end{remark}

\textbf{Algorithm.}
\Cref{alg:beta_estimation} summarizes the MCF algorithm. The simplest version uses one split
of the sample, writing \(N=2M\): estimate the propensity score on the first half and learn the MCF
on the second. Cross-fitting repeats this construction across folds.

\begin{algorithm}[t]
\caption{\textit{Estimating the MCF}}
\label{alg:beta_estimation}
\begin{algorithmic}[1]
\Require Observations \(Z_{1:N}=(X_i,A_i,Y_i)_{i=1}^N\); bounded class \(g(Y;\beta)\).
\Ensure Estimated MCF \(g(\cdot;\hat\beta)\).

\State Estimate \(\hat e(X)=\hat\Pp(A=1\mid X)\) on the training split \(Z_{1:M}\).
\State On the held-out split \(Z_{M+1:2M}\), obtain \(\hat{\beta}\) by maximizing the IPW objective \Cref{eq:ipw-obj} with stochastic gradient ascent.
\State Inspect outcomes with high and low values of \(g(Y;\hat\beta)\); use near neighbors and minimal pairs to interpret the learned feature.
\end{algorithmic}
\end{algorithm}

\begin{remark}[Mini-batch optimization]
\label{rmk:mb}
The objective in \Cref{eq:ipw-obj} is well suited to mini-batch optimization. Instead of computing the optimization objective over the entire second half of the sample at each iteration, we approximate the full-sample gradient using randomly selected subsets of observations. This introduces stochasticity into the updates, reduces the risk of overfitting, and substantially lowers the computational cost per iteration. For the same reason, different subsets of the data can be used to evaluate different components of the objective. For instance, one may use separate mini-batches to approximate the treated and control terms.
\end{remark}

\begin{remark}[Heterogeneous MCF]
\label{rmk:hetero-est}
To learn a heterogeneous MCF, replace \(g(Y_i;\beta)\) by
\(g(Y_i,X_i;\beta)\) in \Cref{eq:ipw-obj,alg:beta_estimation}. The learned
score is then read as \(g(Y,X;\hat\beta)\): a covariate-dependent feature of the
unstructured outcome.
\end{remark}

Learning MCF from data serves as an exploratory tool for causal inference with unstructured outcomes. It reports which represented features of an
unstructured outcome are most changed by treatment. Interpretation returns to
the training examples: inspect high- and low-scoring outcomes, retrieve near neighbors, look for
minimal pairs, and summarize recurring differences with domain knowledge. The algorithm performs a
search rather than confirmatory inference for a hand-named attribute. Labels such as
``formality,'' ``specificity,'' or ``texture'' are interpretations of the learned score, read from
examples and subsequent summaries.

\subsection{Interpreting the MCF}
\label{sec:mcf-interpret}

The estimator in \Cref{sec:mcf-obs-data} returns a feature-scoring function. A candidate feature-scoring function is
\[
g:\mathcal{Y}\to[0,1],
\]
which assigns each represented outcome \(y\) a number. In the clinical-note example, one such function might assign large values to notes with formal language; another might assign large values to notes with templated headings; another might assign large values to notes with short assessment-and-plan sections; and another might combine several of these patterns. For a fixed feature \(g\), the contrast
\[
\theta(g)
=
\mathbb{E}\{g(Y(1))-g(Y(0))\},
\]
asks how much the average value of that feature changes when the same population is documented with the tool rather than without it. The MCF is the function, within the chosen class, whose contrast is largest.

\textbf{The oracle MCF.}
The population oracle MCF gives a useful starting point for interpretation. Let \(p_a\) denote the density of the potential outcome \(Y(a)\), for \(a\in\{0,1\}\), and define
\[
\Delta(y)=p_1(y)-p_0(y).
\]
Thus \(\Delta(y)>0\) means that represented outcomes near \(y\) are more common under treatment than under control. For any bounded feature \(0\le g\le1\),
\[
\theta(g)
=
\mathbb{E}\{g(Y(1))-g(Y(0))\}
=
\int g(y)\Delta(y)\,dy.
\]
This objective is linear in the value \(g(y)\) at each represented outcome \(y\). Therefore, the value of the oracle MCF can be determined separately at each \(y\). If \(\Delta(y)>0\), assigning \(g(y)=1\) increases the objective as much as possible at that point. If \(\Delta(y)<0\), assigning \(g(y)=0\) avoids adding a negative contribution. Hence every oracle MCF satisfies
\[
g^\star(y)=\mathbf{1}\{\Delta(y)>0\},
\]
up to arbitrary choices on the tie set \(\{y:\Delta(y)=0\}\). The set
\[
\{y:g^\star(y)=1\},
\]
is the part of the represented outcome space that treatment makes more likely.

In the running example, this set may contain several recognizable changes at once: more standardized organization, repeated phrases, more formal diction, shorter plans, or other patterns induced by the documentation tool. The MCF does not name these patterns in advance. It assigns large values to notes whose represented form is more characteristic of the treated potential-outcome distribution. Interpretation comes from inspecting notes with large feature values and describing the recurring differences they share.

The oracle MCF is transparent, but it may be too broad for interpretation. A documentation tool may slightly move many aspects of writing, and \(g^\star\) assigns value one to every represented note type whose probability increases, even if the increase is small. For interpretation, it is often better to begin with the clearest part of the shift. A hospital reviewing tool-assisted notes may first want the notes whose organization, phrasing, compression, or templated structure most strongly marks tool use.

\textbf{Adding a budget.}
This motivates a budgeted version of the MCF. The contrast remains the same, but the function is constrained to assign value one to only a fixed amount of the outcome population. Let \(P_{\mathrm{ref}}\) be a reference distribution on represented outcomes, with density \(p_{\mathrm{ref}}\). This distribution defines the population with respect to which selected mass is measured. A natural choice is the pooled potential-outcome distribution
\[
p_{\mathrm{ref}}(y)
=
\pi p_1(y)+(1-\pi)p_0(y),
\qquad \pi\in(0,1),
\]
where \(\pi\) is the mixture weight. Alternatively, the reference distribution may be the corresponding adjusted empirical distribution of represented outcomes. For a budget \(\tau\in(0,1)\), define
\begin{equation}
\label{eq:budget-mcf}
g^\star_\tau
\in
\operatorname*{arg\,max}_{0\le g\le1}
\int g(y)\Delta(y)\,dy
\quad
\text{subject to}
\quad
\int g(y)p_{\mathrm{ref}}(y)\,dy=\tau.
\end{equation}
The constraint says that the outcomes to which \(g^\star_\tau\) assigns value one have reference mass \(\tau\). A small budget makes \(g^\star_\tau\) assign value one only to the most treatment-enriched notes. A larger budget allows \(g^\star_\tau\) to assign value one to notes with weaker, but still systematic, evidence of the same treatment-induced shift.

The budgeted oracle ranks outcomes by
\[
s(y)=\frac{\Delta(y)}{p_{\mathrm{ref}}(y)}.
\]
We call \(s(y)\) the treatment enrichment score. The numerator measures how much more density treatment places near \(y\). The denominator puts that increase on the scale of the reference population. Thus \(s(y)\) asks: among the kinds of notes that appear in the population being reviewed, which kinds become most disproportionately common when the documentation tool is used?

In the clinical-note example, notes with highly regular headings and compressed assessment-and-plan language may have large positive enrichment scores if they appear much more often under tool use. Notes with mildly more formal transitions may also be more common under treatment, but they appear lower in the ranking when the increase is weaker. The budget~\(\tau\) determines how much of this ordered population the function highlights.

Let
\[
\tau^\star
=
\int \mathbf{1}\{\Delta(y)>0\}p_{\mathrm{ref}}(y)\,dy,
\]
denote the reference mass of the represented outcomes for which treatment increases density. This is the reference mass of the set to which the unbudgeted oracle MCF assigns value one.

\begin{proposition}[Budgeted oracle MCF]
\label{prop:ext-3}
Assume \(p_{\mathrm{ref}}(y)>0\) on the relevant support and \(0<\tau<\tau^\star\). The treatment enrichment random variable is
\[
\frac{\Delta(Y)}{p_{\mathrm{ref}}(Y)}.
\]
Suppose this random variable has no atom at its upper \(\tau\)-threshold under \(P_{\mathrm{ref}}\). Then the budgeted oracle MCF is unique almost surely and satisfies
\begin{equation}
\label{eq:budget-threshold}
g^\star_\tau(y)
=
\mathbf{1}\left\{
\frac{\Delta(y)}{p_{\mathrm{ref}}(y)}
\ge
\lambda_\tau
\right\},
\end{equation}
where \(\lambda_\tau\) is chosen so that
\[
\int g^\star_\tau(y)p_{\mathrm{ref}}(y)\,dy=\tau.
\]
Moreover, \(\theta(g^\star_\tau)\) is increasing in \(\tau\) on \((0,\tau^\star)\).
\end{proposition}

The proof is in \Cref{app:pf-ext-3}. The proposition says that the budgeted MCF assigns value one to the outcomes with the largest treatment enrichment scores, stopping when the total reference mass reaches \(\tau\). As \(\tau\) increases, the threshold \(\lambda_\tau\) decreases, so \(g^\star_\tau\) assigns value one to a larger set of represented outcomes. For \(\tau<\tau^\star\), every newly included outcome still has positive density difference, so the optimal contrast increases.

\textbf{The saturation point of the budget.}
Let
\[
\tau_0
=
\int \mathbf{1}\{\Delta(y)=0\}p_{\mathrm{ref}}(y)\,dy,
\]
denote the reference mass of the zero-difference set. At \(\tau=\tau^\star\), the exact-budget optimizer is
\[
g^\star_{\tau^\star}(y)
=
\mathbf{1}\{\Delta(y)>0\}
\qquad
P_{\mathrm{ref}}\text{-almost surely}.
\]
Even when the zero-difference set has positive reference mass, assigning positive mass to that set would require removing the same mass from the positive-difference set and would decrease the objective.

For \(\tau\in(\tau^\star,\tau^\star+\tau_0]\), the additional budget can be assigned to the zero-difference set without changing the optimal contrast. Only when \(\tau>\tau^\star+\tau_0\) does the equality constraint in \Cref{eq:budget-mcf} force the function to assign positive mass to outcomes with \(\Delta(y)<0\), which lowers the contrast. With an upper-bound budget,
\[
\int g(y)p_{\mathrm{ref}}(y)\,dy\le \tau,
\]
the optimal value is attained by an optimizer that stops once all outcomes with positive density difference have been assigned value one. Thus \(\tau^\star\) is the beginning of the saturation region: the exact-budget value increases before \(\tau^\star\), remains constant through \(\tau^\star+\tau_0\), and decreases once negative-difference outcomes must be included.

In the clinical-note example, a one-percent budget may highlight notes whose tool signature is unmistakable: standardized organization, repeated templated language, or unusually compressed assessment-and-plan sections. A ten-percent budget may include milder but still systematic shifts: more regular headings, more formal transitions, shorter plan descriptions, or phrasing that resembles the clearest selected notes less strongly. As the budget approaches \(\tau^\star\), the set of notes assigned value one expands to all represented note types made more common by the tool. Once the budget exceeds \(\tau^\star+\tau_0\), an exact budget would begin to include note types more characteristic of unaided documentation.

\textbf{Interpreting a fitted MCF.}
The fitted MCF may combine several human-interpretable attributes. A note may receive a high value of \(g(Y)\) because it is formal, concise, templated, unusually structured, technical, or some mixture of these. An attribute with little direct mean shift may still help identify treated-like notes when it is correlated with an attribute that shifts strongly. For example, politeness may help distinguish tool-assisted notes because it co-occurs with formality, even when the main tool-induced change is formality. The causal contrast belongs to the learned feature score~\(g(Y)\) as a whole.

\textbf{A simple Gaussian two-feature calculation.}
A simple Gaussian calculation makes this point concrete. Suppose a note representation is summarized by two standardized attributes, say formality and politeness. Let
\[
Y(a)\sim \mathcal N(\mu_a,\Sigma),
\qquad
\delta=\mu_1-\mu_0,
\]
with a common covariance matrix \(\Sigma\). In this equal-covariance Gaussian case, the likelihood ratio is monotone in a single linear score,
\[
S(Y)=w^\top Y,
\qquad
w=\Sigma^{-1}\delta.
\]
Thus the oracle MCF can be written as a function of \(S(Y)\). Notes with larger values of \(S(Y)\) are notes that move farther in the direction treatment tends to move the representation. If formality has the larger standardized treatment shift, the formality coefficient has the larger magnitude after standardization. If politeness changes mainly through its correlation with formality, its coefficient reflects the remaining covariance-adjusted contribution. Standardization puts the attributes on a common marginal scale, so the comparison is not driven by units.

The budget chooses a cutoff on this same score. With a very small budget, \(g^\star_\tau\) assigns value one only to notes that sit farthest in the treatment direction, such as notes that are both highly formal and strongly templated. With a larger budget, \(g^\star_\tau\) may also assign value one to notes that are moderately formal, or notes that carry correlated markers such as regular headings, compressed plans, or polished transitions. These notes are included because they lie farther down the same treatment-enrichment ordering. We detail the calculation in \Cref{sec:mcf-learn}.

\textbf{Covariate-dependent features.}
The same logic extends to covariate-dependent features. Let \(X\) denote the baseline covariates taking values in the covariate space~\(\mathcal X\), let \(p(x)\) denote their marginal density, let \(p_a(y\mid x)\) be the conditional density of \(Y(a)\) given \(X=x\), and define
\[
\Delta(y\mid x)=p_1(y\mid x)-p_0(y\mid x).
\]
A covariate-dependent feature is a function
\[
g:\mathcal{Y}\times\mathcal{X}\to[0,1],
\]
whose value may depend on both the represented note and the covariate profile. This is useful when treatment changes different aspects of notes in different subpopulations. For instance, the documentation tool may make routine-visit notes more templated while making complex-visit notes more complete.

For any such feature, the covariate-adaptive contrast is
\[
\theta_X(g)
=
\mathbb{E}\{g(Y(1),X)-g(Y(0),X)\}
=
\int g(y,x)p(x)\Delta(y\mid x)\,dy\,dx.
\]
The unbudgeted heterogeneous oracle MCF is
\[
g^\star_X(y,x)=\mathbf{1}\{\Delta(y\mid x)>0\},
\]
up to arbitrary choices on the tie set \(\{(y,x):\Delta(y\mid x)=0\}\). In the clinical-note example, the set of notes assigned value one may differ across visit types, clinicians, or patient groups.

For the budgeted heterogeneous problem, let \(P_{\mathrm{ref}}\) be a reference distribution on \((Y,X)\), with density \(p_{\mathrm{ref}}(y,x)\). A natural pooled choice is
\[
p_{\mathrm{ref}}(y,x)
=
p(x)\{\pi p_1(y\mid x)+(1-\pi)p_0(y\mid x)\}.
\]
The budgeted heterogeneous oracle solves
\[
g^\star_{X,\tau}
\in
\operatorname*{arg\,max}_{0\le g\le1}
\int g(y,x)p(x)\Delta(y\mid x)\,dy\,dx
\quad
\text{subject to}
\quad
\int g(y,x)p_{\mathrm{ref}}(y,x)\,dy\,dx=\tau.
\]
Define the reference mass of the positive conditional-density-difference set as
\[
\tau^\star_X
=
\int
\mathbf{1}\{\Delta(y\mid x)>0\}
p_{\mathrm{ref}}(y,x)\,dy\,dx.
\]

\begin{proposition}[Heterogeneous budgeted oracle MCF]
\label{prop:ext-4}
Assume \(p_{\mathrm{ref}}(y,x)>0\) and \(p(x)>0\) on the relevant support and \(0<\tau<\tau^\star_X\). The covariate-specific treatment enrichment random variable is
\[
\frac{p(X)\Delta(Y\mid X)}{p_{\mathrm{ref}}(Y,X)}.
\]
Suppose this random variable has no atom at its upper \(\tau\)-threshold under \(P_{\mathrm{ref}}\). Then the budgeted heterogeneous oracle MCF is unique almost surely and satisfies
\begin{equation}
\label{eq:hetero-budget}
g^\star_{X,\tau}(y,x)
=
\mathbf{1}\left\{
\frac{p(x)\Delta(y\mid x)}
{p_{\mathrm{ref}}(y,x)}
\ge
\lambda_\tau
\right\},
\end{equation}
where \(\lambda_\tau\) is chosen so that
\[
\int g^\star_{X,\tau}(y,x)p_{\mathrm{ref}}(y,x)\,dy\,dx=\tau.
\]
Moreover, \(\theta_X(g^\star_{X,\tau})\) is increasing in \(\tau\) on \((0,\tau^\star_X)\).
\end{proposition}

The proof is in \Cref{app:pf-ext-4}. The score inside \Cref{eq:hetero-budget} is the covariate-specific treatment enrichment score. Under the pooled reference distribution above,
\[
\frac{p(x)\Delta(y\mid x)}{p_{\mathrm{ref}}(y,x)}
=
\frac{\Delta(y\mid x)}
{\pi p_1(y\mid x)+(1-\pi)p_0(y\mid x)}.
\]
For a fixed covariate profile \(x\), this score is large when notes near \(y\) become much more common under treatment than under control, relative to how often such notes appear among units with that covariate profile. The single threshold \(\lambda_\tau\) lets the budget be spent where the treatment-induced difference is sharpest across the joint space of notes and covariates. In the running example, routine-visit notes assigned value one may be those with templated structure, while complex-visit notes assigned value one may be those with fuller histories or more explicit plans.

Define the reference mass of the zero conditional-density-difference set as
\[
\tau_{0,X}
=
\int \mathbf{1}\{\Delta(y\mid x)=0\}p_{\mathrm{ref}}(y,x)\,dy\,dx.
\]
At \(\tau=\tau^\star_X\), the exact-budget optimizer assigns value one to all \((y,x)\) pairs with positive conditional density difference and value zero to the zero-difference set. For \(\tau\in(\tau^\star_X,\tau^\star_X+\tau_{0,X}]\), additional mass can be assigned to zero-difference pairs without changing the optimal contrast. Negative-difference pairs are required only when \(\tau>\tau^\star_X+\tau_{0,X}\). With an upper-bound budget, the optimum saturates at \(\tau^\star_X\).

\textbf{Using the budget in practice.}
In practice, one fits the MCF, chooses several budgets, and inspects the outcomes selected at each budget. For clinical notes, this means comparing selected notes with low-scoring or reference notes, retrieving near neighbors, constructing minimal pairs when possible, and asking domain experts to describe recurring differences. These descriptions can then be summarized with hand-coded variables, structured annotations, or follow-up confirmatory analyses. The MCF supplies an exploratory map of where the treatment-induced contrast is concentrated in the represented outcome space.

\section{When both treatments and outcomes are unstructured}
\label{sec:extension-hd-treatment-outcome}

The preceding sections study a setting in which the treatment is simple and the outcome is complex. A clinician either uses an AI documentation tool or does not; the object of interest is the note that is eventually written. In many applications, the same difficulty appears on both sides of the causal relationship. The treatment is itself a document, image, plan, prompt, or other unstructured object, and the outcome is another unstructured object.

We continue with the clinical-documentation example. During a visit, an AI system may show the clinician a draft suggestion: a paragraph, an assessment, a list of questions, or a proposed plan. The treatment $A$ is this suggestion. The outcome $Y$ is the final note signed by the clinician. Both $A$ and $Y$ are unstructured text objects. The covariates $X$ include the visit type, patient complexity, clinician, specialty, time pressure, and other recorded factors that may shape both the suggestion produced by the system and the note eventually written.

The causal question is which aspects of the suggestion are associated with which aspects of the final note, after comparing visits with similar observed covariates. Do suggestions that contain concrete assessment-and-plan language lead to notes with more complete follow-up plans? Do highly templated suggestions lead to more templated notes? Do patient-specific suggestions lead to final notes that retain more patient-specific details? These are causal questions about directions in two unstructured spaces.

\textbf{A two-sided feature.}
The MCF applies a feature-scoring function to an unstructured outcome, with the function chosen by a causal contrast. Here we need two such functions. We learn a treatment score~$f(A)$ and an outcome score~$g(Y)$. The treatment score marks the aspect of the suggestion being studied; the outcome score marks the aspect of the final note being studied. The pair~$(f,g)$ is chosen so that, among comparable visits, high values of the treatment score line up with high values of the outcome score.

A related treatment-side problem was studied by \citet{wibisono2026causal}, where the treatment is unstructured and the outcome is scalar. The present setting adds the outcome-side difficulty: the scalar outcome feature is also learned from the data. The object of interest is therefore a matched pair of feature-scoring functions. Below we first consider global functions~$f(A)$ and~$g(Y)$, whose meanings do not vary with~$X$. We then give the covariate-adaptive version, where the same suggestion or note may be scored differently at different covariate profiles.

\subsection{Maximally influential treatment features and maximally contrasting outcome features}
\label{sec:hdto-homogeneous}

Let $A$ denote the treatment object, $Y$ the outcome object, and $X$ the observed confounders. We learn two bounded feature-scoring functions,
\[
    f(A)\in(\epsilon,1-\epsilon),
    \qquad
    g(Y)\in(0,1),
\]
for a fixed lower-bound constant~$\epsilon\in(0,1/2)$. The bounds put both scores on a stable scale and prevent a degenerate treatment score at exactly zero or one. In the documentation example, a high value of $f(A)$ may mean that the AI suggestion contains a concrete clinical plan, while a high value of $g(Y)$ may mean that the final note contains a complete follow-up plan. In practice, we often parametrize them as
\[
    f_\gamma(A)
    =
    \epsilon+(1-2\epsilon)\sigma\{F_\gamma(A)\},
    \qquad
    g_\beta(Y)
    =
    \sigma\{G_\beta(Y)\},
\]
where $F_\gamma$ and $G_\beta$ are neural networks with trainable parameter vectors~$\gamma$ and~$\beta$, respectively, and $\sigma$ is the sigmoid function. The population definitions below do not depend on this parameterization.

\textbf{The covariate baseline.}
For a fixed outcome feature $g$, define
\[
    m_g(x)=\E\{g(Y)\mid X=x\}.
\]
This is the average value of the learned outcome score among units with covariates $x$. In the running example, $m_g(x)$ is the average value of the note feature among visits with the same recorded characteristics. The residual
\[
    g(Y)-m_g(X),
\]
measures whether the observed final note is high on the learned feature relative to comparable visits.

\textbf{MIF--MCF.} To identify the maximally influential features (MIF) of the treatment and the maximally contrasting features (MCF) of the outcome, we choose $f$ and $g$ by maximizing
\begin{equation}
\label{eq:hdto-objective}
    \mathcal{J}(f,g)
    =
    \E\!
    \left[
        f(A)\{g(Y)-m_g(X)\}
    \right].
\end{equation}
The criterion rewards a treatment feature that gives high scores to treatment objects whose paired outcomes are high after covariate adjustment. If, among comparable visits, concrete AI suggestions are followed by unusually complete final notes, then a treatment score measuring concreteness and an outcome score measuring completeness increase the objective.

\textbf{Within-covariate association.}
The same objective can be read symmetrically. It is the average conditional covariance between the learned treatment score and the learned outcome score.
\begin{proposition}[Equivalent forms of the joint feature objective]
\label{prop:hdto-equivalent}
For any measurable $f$ and $g$ such that the expectations exist,
\[
\begin{aligned}
\E\!
\left[
    f(A)\{g(Y)-\E(g(Y)\mid X)\}
\right]
&\;=\;
\E\!
\left[
    \{f(A)-\E(f(A)\mid X)\}
    \{g(Y)-\E(g(Y)\mid X)\}
\right]
\\
&\;=\;
\E\!
\left[
    \operatorname{Cov}\{f(A),g(Y)\mid X\}
\right].
\end{aligned}
\]
\end{proposition}
The proof is in \Cref{app:proof_hdto_equivalent}. This identity is the adjustment step in a compact form. The comparison is made within covariate strata and then averaged over the population. In the documentation example, the method compares suggestions and notes after removing the part of the note feature predictable from the visit, clinician, patient, and other observed covariates.

\subsection[Estimating the MIF--MCF pair by matched negative control outcomes]{Estimating the MIF--MCF pair by matched negative control outcomes}
\label{sec:hdto-estimation}

The population objective contains the function~$m_g(x)=\E\{g(Y)\mid X=x\}$. This function changes whenever~$g$ changes. Regressing the current value of~$g(Y)$ on~$X$ at every training step would put one learning problem inside another. We instead estimate the covariate baseline by comparing each observed outcome with matched negative-control outcomes.

\textbf{Negative control outcomes.}
Let $Y^\dagger$ be a negative control outcome satisfying
\[
    Y^\dagger\mid X \sim Y\mid X,
    \qquad
    Y^\dagger \perp A \mid X.
\]
Thus $Y^\dagger$ has the same covariate-specific distribution as $Y$, but carries no remaining association with the treatment object once $X$ is fixed. In the running example, a negative outcome for a visit is a final note from another visit with similar recorded characteristics, rather than the note paired with the AI suggestion actually shown.

\begin{proposition}[Negative control outcomes recover the centered objective]
\label{prop:negative-outcome-hdto}
For any measurable $f$ and $g$, we have
\[
    \E\{f(A)g(Y)\}
    -
    \E\{f(A)g(Y^\dagger)\}
    =
    \E\!
    \left[
        f(A)\{g(Y)-\E(g(Y)\mid X)\}
    \right].
\]
\end{proposition}
The proof is in \Cref{app:proof_negative-outcome-hdto}.

\textbf{The empirical objective.}
In finite samples, exact draws from $Y\mid X=X_i$ are rarely available. We approximate them by matching. Let~$n$ be the sample size. For each observation~$i$, let~$N(i)$ be a neighborhood of observations with covariates close to~$X_i$, and let the nonnegative matching weights~$w_{ij}$ satisfy~$\sum_{j\in N(i)}w_{ij}=1$. For example, $N(i)$ may contain the~$K$ nearest neighbors of~$X_i$, where~$K$ is the neighborhood size. With a kernel function~$\mathcal K$ and bandwidth~$h$, the weights may be
\[
w_{ij}=\frac{\mathcal K(\|X_i-X_j\|/h)}{\sum_{\ell\in N(i)}\mathcal K(\|X_i-X_\ell\|/h)},\qquad j\in N(i).
\]
For the current outcome feature $g_\beta$, the matched baseline is
\[
    \widehat b_\beta(X_i)
    =
    \sum_{j\in N(i)} w_{ij}g_\beta(Y_j).
\]
We thus estimate $(\gamma,\beta)$ by maximizing
\begin{equation}
\label{eq:hdto-empirical}
    \widehat{\mathcal{J}}(\gamma,\beta)
    =
    \frac{1}{n}
    \sum_{i=1}^n
    f_\gamma(A_i)
    \left\{
        g_\beta(Y_i)
        -
        \sum_{j\in N(i)} w_{ij}g_\beta(Y_j)
    \right\}.
\end{equation}
The matching weights depend only on $X$ and are fixed during training. The two networks can therefore be updated together, while the adjustment set remains fixed.

After training, high values of $f_{\hat\gamma}(A)$ mark treatment objects associated with positive residualized values of the learned outcome feature. High values of $g_{\hat\beta}(Y)$ mark outcome objects associated with positive residualized values of the learned treatment feature. In the running example, the learned pair can be inspected by comparing high- and low-scoring AI suggestions, and high- and low-scoring final notes.

\begin{remark}[Exact centering]
When the covariate baseline can be computed directly, negative outcomes are unnecessary. If $X$ is discrete with repeated values, define
\[
    \widehat m_\beta(x)
    =
    \frac{
        \sum_{j=1}^n \mathds{1}(X_j=x)g_\beta(Y_j)
    }{
        \sum_{j=1}^n \mathds{1}(X_j=x)
    }.
\]
Then one can optimize
\[
    \widehat{\mathcal{J}}_{\mathrm{exact}}(\gamma,\beta)
    =
    \frac{1}{n}
    \sum_{i=1}^n
    f_\gamma(A_i)
    \{g_\beta(Y_i)-\widehat m_\beta(X_i)\}.
\]
Here $\widehat m_\beta(X_i)$ is recomputed as $g_\beta$ changes, while the optimization remains over $(\gamma,\beta)$.
\end{remark}

To interpret the learned MIF--MCF pair, we follow the approach in \Cref{sec:mcf-interpret}. \Cref{sec:hdto-selection} provides a detailed analysis of the MIF--MCF objective, and \Cref{sec:hdto-heterogeneous} extends the construction to heterogeneous effects.

\section{Empirical studies}
\label{sec:experiments}

We illustrate the MCF algorithm using synthetic and semi-synthetic experiments. The empirical studies ask whether the MCF identifies the represented aspect of an unstructured outcome that changes under treatment, whether a covariate-adaptive MCF recovers context-dependent changes, and whether paired feature-scoring functions recover linked directions when both the treatment and the outcome are unstructured. Experiment~I studies text-based outcomes, Experiment~II studies image-based outcomes, and Experiment~III considers settings in which both treatments and outcomes are unstructured.

We find the following. For text-based outcomes, the MCF recovers global and context-dependent changes in formality and non-toxicity and isolates the textual dimensions changed by treatment. For image-based outcomes, the learned feature score separates low- and high-blur cell images, and nudging increases blur while holding cell-count content fixed. When both treatment and outcome are unstructured, paired feature-scoring functions isolate linked treatment--outcome coordinates and recover prompt-to-headline formality after adjusting for topic.

\subsection{Experiment I: Empirical studies on text-based outcomes}
\label{sec:exp-text-outcome}

We first evaluate the MCF algorithm when the outcome is text. The goal is to learn an interpretable feature-scoring function of the text outcome that is shifted by treatment after adjusting for context. The formality study examines treatment-induced formality and compares the learned feature-scoring function with a topic-model-based alternative. The detoxification study considers a safety-relevant setting in which treatment may induce more or less non-toxic language depending on context. The third study asks whether the MCF isolates the exact combination of textual dimensions changed simultaneously by treatment.

\subsubsection{Treatment-induced formality in text outcomes}
\label{sec:formal-gyafc}

\textbf{Causal question.}
We begin with a semi-synthetic experiment using the GYAFC dataset~\citep{rao2018dear}, motivated by text-generation systems used on online platforms. Suppose a platform deploys an editing or writing-assistance tool that can affect the style of user-facing text, such as comments, replies, or short posts. The platform observes the context in which the text is written, such as whether it concerns the GYAFC domains Entertainment \& Music or Family \& Relationships, and wants to understand which aspect of the generated text changes under the intervention. The treatment~$A$ indicates whether a formality-inducing intervention is applied, the outcome~$Y$ is the resulting text, and the context~$X$ records the topic. The causal question is: \textit{which textual feature of the outcome is induced by treatment after accounting for the context in which the text is produced?}

\textbf{Experimental setup.}
The binary style label~$F$ indicates whether a sentence is formal. Treatment assignment depends on context, and context also affects the outcome-text distribution; the context is therefore a confounder.

\textbf{Treatment scenarios.}
In Scenario~1, treatment increases formality in both contexts, so the target is a homogeneous feature-scoring function~$g(Y)$ that assigns high scores to formal sentences and low scores to informal sentences. In Scenario~2, treatment decreases formality in entertainment contexts and increases it in family contexts. This mechanism represents an intervention that may encourage a conversational style in entertainment contexts but a formal style in family or advice-related contexts. The target is therefore a covariate-adaptive feature-scoring function~$g(Y,X)$ that captures the treatment direction relative to the context-specific baseline, rather than a global formality classifier.

\textbf{Estimation.}
We estimate both functions using the proposed objective. Text outcomes are embedded before training, and all summaries use a held-out validation set. The complete variable definitions, assignment mechanism, and outcome mechanisms are in \Cref{app:emp-formality}.

\textbf{Results.}
\Cref{fig:formal_heatmap} shows that the learned scores follow the treatment-induced direction. In Scenario~1, the scores are low for informal sentences and high for formal sentences in both contexts. The average feature scores for informal and formal entertainment sentences are approximately $0.19$ and $0.86$, while the corresponding scores for family sentences are approximately $0.21$ and $0.86$. Thus, the homogeneous MCF recovers formality as the treatment-induced textual direction.

In Scenario~2, the learned feature-scoring function adapts to context. Informal and formal entertainment sentences have average scores of approximately $0.68$ and $0.10$. In the family context, the pattern reverses: informal and formal sentences have average scores of approximately $0.23$ and $0.72$. The covariate-adaptive MCF therefore does not simply learn a global formality classifier; it identifies the outcome-side feature associated with the treatment direction after adjusting for context.

\begin{figure}[t]
\centering
\begin{subfigure}[b]{0.47\textwidth}
\centering
\includegraphics[width=\textwidth]{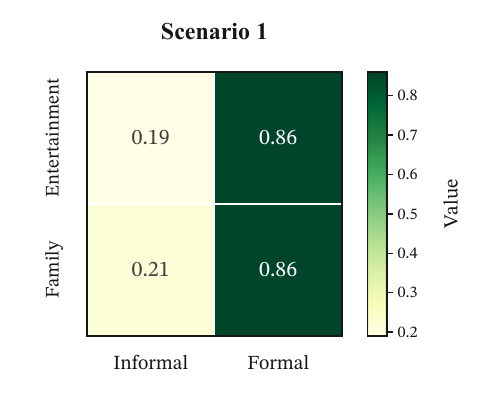}
\caption{Scenario 1}
\end{subfigure}
\hfill
\begin{subfigure}[b]{0.47\textwidth}
\centering
\includegraphics[width=\textwidth]{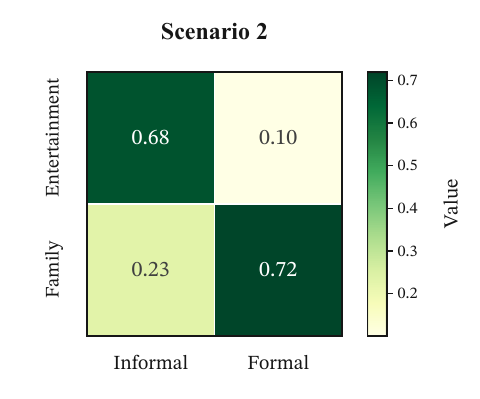}
\caption{Scenario 2}
\end{subfigure}
\caption{The learned feature scores recover the global and context-dependent formality directions. Average learned outcome scores by context and style label in the formality experiment. In Scenario 1, the learned feature assigns higher values to formal sentences in both contexts. In Scenario 2, the learned feature is context-dependent: informal sentences receive higher values in the entertainment context, while formal sentences receive higher values in the family context.}
\label{fig:formal_heatmap}
\end{figure}

Embedding-based nudging~\citep{wibisono2026causal} checks whether increasing the learned feature score yields the transformation predicted by the data-generating process. In Scenario~1, the average formality score increases monotonically from approximately $0.15$ to $0.82$ over ten iterations. In Scenario~2, nudging initially informal family text increases average formality from approximately $0.28$ to $0.63$, while nudging initially formal entertainment text decreases average formality from approximately $0.84$ to $0.72$. Thus, nudging follows the context-specific treatment direction rather than imposing a global increase in formality. The full low- and high-score examples, qualitative interpretation, procedure, and complete trajectories are in \Cref{app:emp-formality}.

\Cref{tab:formal-nudging-examples} retains concrete text examples in the main paper. The nudged sentences are expected to preserve the broad content of the originals while shifting toward a more formal register. These examples provide a qualitative check that the learned feature-scoring function corresponds to a meaningful and interpretable text transformation.

\begin{table}[t]
    \centering
    \small
    \begin{tabularx}{\textwidth}{p{0.09\textwidth}>{\raggedright\arraybackslash}X>{\raggedright\arraybackslash}X}
        \toprule
        Example & Original text & Nudged text \\
        \midrule
        1 & \textit{I HAVE SO MANY I'LL CHOOSE WAT U KNOW ABOUT THAT HAVE U EVER HEARD OF DAT!!!!!!!!!} & \textit{I have so many I choose Those I know of You have heard of this. Is it?} \\
        2 & \textit{MARRIAGE WONT CHANGE A THING EITHER WAY!} & \textit{Marriage Will Change in Another Way.} \\
        3 & \textit{What I like about the 5-string is that I don't ever re-tune it.} & \textit{The thing I like about the five-string is that I do not readjust it.} \\
        4 & \textit{(And if she comes up with any other answer, just tell her no!)} & \textit{If she comes up with any other answer, just say no.} \\
        5 & \textit{what kaitlyn said im sorry im no help its my job to disappoint people} & \textit{What Kaitlyn said is sorry. I am helpless. It is my job to disappoint people.} \\
        \bottomrule
    \end{tabularx}
    \caption{Decoded nudges express the learned global formality direction as more formal text. Examples of original and nudged text outcomes in Scenario 1. Nudging increases the learned outcome feature and produces more formal text.}
    \label{tab:formal-nudging-examples}
\end{table}

\textbf{Comparison with topic modeling.}
We instantiate the codebook framework of \citet{egami2022make} with latent Dirichlet allocation~\citep{blei2003latent}, focusing on Scenario~1, the context-invariant case. The comparison fits models with $K\in\{2,\ldots,10\}$ topics and searches for the subset of dominant topics with the largest causal contrast. Comparing the topics inside and outside this subset is intended to reveal which aspects of the text are causally affected. The model with $K=8$ has the largest contrast, but its selected and unselected topics do not map directly to formality, sentiment, or another clear stylistic dimension. The complete construction, all representative topic words, and the interpretation of this limitation are in \Cref{app:emp-formality-lda}.

\subsubsection{Treatment-induced toxicity in text outcomes}
\label{sec:toxicity}

\textbf{Causal question.}
The second study is a safety-relevant text-outcome experiment using the ParaDetox dataset~\citep{logacheva2022paradetox}, which contains toxic sentences paired with human-written non-toxic paraphrases. It follows the structure of the GYAFC study but replaces formality with non-toxicity. The motivating platform deploys a moderation or reply-assistance intervention that may affect the safety of user-facing text and observes whether the comment appears in a high-conflict discussion or a moderated support setting. The causal question is: \textit{which feature of the outcome is induced by treatment after accounting for the context in which the text is produced?} We define the external non-toxicity score as one minus the toxicity score returned by Detoxify~\citep{Detoxify}.

\textbf{Experimental setup.}
The binary context~$X$ distinguishes high-conflict discussions from moderated support settings, and the binary label~$T$ distinguishes toxic from non-toxic text. Context affects both treatment assignment and the outcome-text distribution and therefore acts as a confounder.

\textbf{Treatment scenarios.}
Scenario~1 makes treatment increase non-toxicity in both contexts. Its homogeneous feature-scoring function~$g(Y)$ should assign high scores to non-toxic sentences and low scores to toxic sentences. Scenario~2 makes treatment decrease non-toxicity in high-conflict discussions and increase it in moderated support settings. It represents an intervention that may amplify confrontational language in one setting while encouraging safer or more neutral language in the other. Its covariate-adaptive feature-scoring function~$g(Y,X)$ should therefore recover the treatment direction relative to the context-specific baseline, rather than a global toxicity classifier.

\textbf{Estimation.}
Text outcomes are embedded before training, and all reported summaries use a held-out validation set. The complete variable definitions, assignment mechanism, and outcome mechanisms are in \Cref{app:emp-toxicity}.

\textbf{Results.}
\Cref{fig:toxic_heatmap} shows that Scenario~1 recovers global non-toxicity. The learned scores are low for toxic sentences and high for non-toxic sentences in both contexts. In the high-conflict context, toxic and non-toxic sentences have average scores of approximately $0.05$ and $0.92$. In the moderated support context, the corresponding scores are approximately $0.03$ and $0.91$. Thus, the learned feature-scoring function recovers non-toxicity as the treatment-induced textual direction.

In Scenario~2, the learned feature-scoring function adapts to context. In the high-conflict context, toxic and non-toxic sentences have average scores of approximately $0.84$ and $0.06$. In the moderated support context, the pattern reverses: the corresponding scores are approximately $0.04$ and $0.82$. Thus, the algorithm does not merely learn a global toxicity or non-toxicity score; it identifies the outcome-side feature associated with the treatment direction after adjusting for context.

\begin{figure}[t]
\centering
\begin{subfigure}[b]{0.47\textwidth}
\centering
\includegraphics[width=\textwidth]{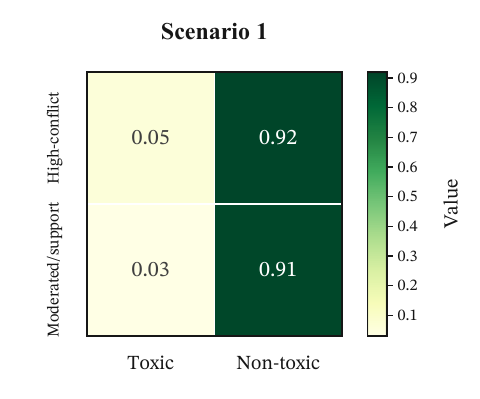}
\caption{Scenario 1}
\end{subfigure}
\hfill
\begin{subfigure}[b]{0.47\textwidth}
\centering
\includegraphics[width=\textwidth]{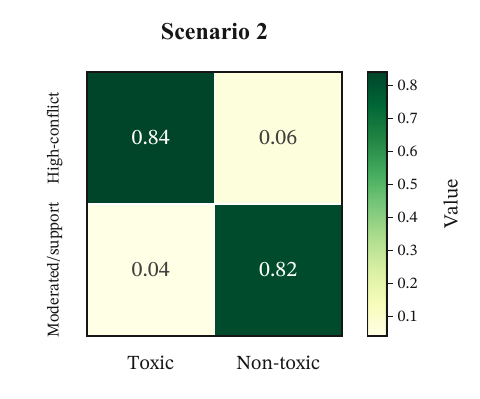}
\caption{Scenario 2}
\end{subfigure}
\caption{The learned feature scores recover global non-toxicity and the opposing context-dependent directions. Average learned outcome scores by context and toxicity label in the toxicity experiment. In Scenario 1, the learned feature assigns higher values to non-toxic sentences in both contexts. In Scenario 2, the learned feature is context-dependent: toxic sentences receive higher values in the high-conflict context, while non-toxic sentences receive higher values in the moderated support setting.}
\label{fig:toxic_heatmap}
\end{figure}

Nudging again checks the learned direction. In Scenario~1, the average non-toxicity score increases monotonically from approximately $0.04$ to $0.77$ over ten iterations. In Scenario~2, nudging initially toxic moderated-support text increases average non-toxicity from approximately $0.05$ to $0.72$, while nudging initially non-toxic high-conflict text decreases average non-toxicity from approximately $0.96$ to $0.83$. The learned transformation therefore follows the context-dependent direction induced by treatment, rather than imposing a global push toward safer language. The full low- and high-score examples, qualitative interpretation, procedure, and complete trajectories are in \Cref{app:emp-toxicity}.

\Cref{tab:toxic-nudging-examples} retains representative text transformations in the main paper. The nudged comments are expected to preserve the broad content of the originals while shifting toward less toxic language. These examples provide a qualitative check that the learned feature-scoring function corresponds to a meaningful and interpretable safety-related text transformation.

\begin{table}[t]
    \centering
    \small
    \begin{tabularx}{\textwidth}{p{0.09\textwidth}>{\raggedright\arraybackslash}X>{\raggedright\arraybackslash}X}
        \toprule
        Example & Original text & Nudged text \\
        \midrule
        1 & \textit{did you use that diploma to wipe your ass with yet ?} & \textit{Did you use that with to wipe yourself with yet?} \\
        2 & \textit{the writing is totally fucked up !} & \textit{The writing is completely terrible!} \\
        3 & \textit{delete the page and shut up	} & \textit{Remove the page and shut} \\
        4 & \textit{wow you 're super awesome bitch} & \textit{Wow you're super awesome girl} \\
        5 & \textit{what the fuck is going on ? ? !} & \textit{What's happening?} \\
        \bottomrule
    \end{tabularx}
    \caption{Decoded nudges express the learned safety direction as less toxic text. Examples of original and nudged text outcomes in Scenario 1. Nudging increases the learned outcome feature and produces less toxic text.}
    \label{tab:toxic-nudging-examples}
\end{table}

\subsubsection{Multiple treatment-induced textual dimensions}
\label{sec:multiple-dimension}

\textbf{Causal question.}
The previous studies consider settings in which treatment primarily affects one interpretable textual dimension, such as formality or non-toxicity. In practice, however, a text intervention may change several aspects of the outcome simultaneously. For example, a writing-assistance tool may make a response more formal while also changing its punctuation, emotional intensity, or conversational style. This study asks whether the MCF recovers the specific combination of textual attributes induced by treatment, rather than merely recovering a pre-specified label. The causal question is: \textit{which combination of textual attributes is induced by treatment?}

\textbf{Experimental setup.}
The binary attribute~$F$ records formality, and the binary attribute~$P$ records the presence of a question mark or exclamation mark. The study uses the GYAFC data~\citep{rao2018dear} and the same treatment-assignment mechanism as the formality study, changing only the outcome mechanism. In Scenario~1, treatment changes both formality and punctuation, so the relevant feature combines both attributes. In Scenario~2, treatment changes punctuation but not formality, so the relevant feature is punctuation alone. In Scenario~3, treatment changes formality but not punctuation, so the relevant feature is formality alone. Texts are represented using the StyleDistance algorithm of \citet{patel2024styledistance}. The complete variable definitions, conditional-independence assumption, and scenario probabilities are in \Cref{app:emp-multiple-dimensions}.

\textbf{Results.}
\Cref{tab:mult-dim} contains the complete results. In Scenario~1, the largest score, $0.940$, occurs for text that is both formal and contains a question or exclamation mark. Texts containing only one of the two treatment-induced attributes receive smaller scores, and texts containing neither receive the smallest score. Thus, the learned feature-scoring function captures the joint textual direction induced by treatment.

In Scenario~2, punctuated texts receive high scores regardless of formality: $0.954$ for $(F,P)=(0,1)$ and $0.935$ for $(F,P)=(1,1)$. Texts without punctuation receive low scores regardless of formality. Thus, the learned feature-scoring function correctly ignores formality and recovers punctuation as the treatment-induced dimension.

In Scenario~3, the pattern reverses. Formal texts receive high scores regardless of punctuation: $0.962$ for $(F,P)=(1,0)$ and $0.969$ for $(F,P)=(1,1)$. Informal texts receive low scores whether or not they contain question or exclamation marks. Thus, the learned feature-scoring function recovers formality as the treatment-induced dimension and ignores punctuation. Overall, the proposed objective recovers different outcome-side textual features according to how treatment changes the text distribution.

\begin{table}[t]
\centering
\caption{The MCF recovers the formality~($F$) and punctuation~($P$) dimensions changed in each scenario. Average $\hat{g}(Y)$ values for each text type across all three scenarios. In each scenario, the learned $\hat{g}$ values are consistent with the corresponding data generating process.}
\begin{tabular}{ccccc}
\toprule
& $F=0$, $P=0$ & $F=0$, $P=1$ & $F=1$, $P=0$ & $F=1$, $P=1$ \\
\midrule
\textit{Scenario 1} & $0.000$ & $0.119$ & $0.204$ & $\mathbf{0.940}$ \\
\textit{Scenario 2} & $0.091$ & $\mathbf{0.954}$ & $0.081$ & $\mathbf{0.935}$ \\
\textit{Scenario 3} & $0.036$ & $0.043$ & $\mathbf{0.962}$ & $\mathbf{0.969}$ \\
\bottomrule
\end{tabular}
\label{tab:mult-dim}
\end{table}

\FloatBarrier
\subsection{Experiment II: Empirical studies on image-based outcomes}
\label{sec:exp-image-outcome}

The second experiment asks whether the MCF isolates a treatment-induced style direction when the outcome is an image.

\textbf{Causal question.}
The study uses synthetic cell-body-stain images from the BBBC005v1 dataset in the Broad Bioimage Benchmark Collection~\citep{ljosa2012annotated}, restricting attention to images with at most $40$ cells. A common imaging question is whether an acquisition or processing intervention changes image style, such as blur or focus, after accounting for content. The causal question is: \textit{does the intervention change focus blur after accounting for the number of cells in the image?}

\textbf{Experimental setup.}
We define \textit{content} as the number of cells and \textit{style} as every image feature other than the number of cells; focus blur is the style of interest. An adversarial autoencoder separates content and style using a $32$-dimensional style representation. Treatment assignment depends on whether the cell count is below $20$, and treatment changes a target blur level used to sample an image from the appropriate cell-count group. We learn the MCF~$g(Y)$ from observational data, where the represented outcome~$Y$ is the image-style embedding. The complete representation, target-blur mechanism, truncation, and nearest-image sampling procedure are in \Cref{app:emp-image-representation}. The complete style-transfer grid and its interpretation are in \Cref{app:emp-image-examples}.

\textbf{Results.}
\Cref{fig:scatter-image-cell} shows a positive relationship between observed blur and the learned feature score in both training and test data. In the test set, images with $\hat g(Y)<0.2$ have mean blur $10.0$, while images with $\hat g(Y)>0.8$ have mean blur $34.6$. Thus, the MCF separates low- and high-blur styles after adjustment for cell-count content.

\begin{figure}[H]
    \centering
    \begin{subfigure}[b]{0.45\textwidth}
        \centering
        \includegraphics[width=\textwidth]{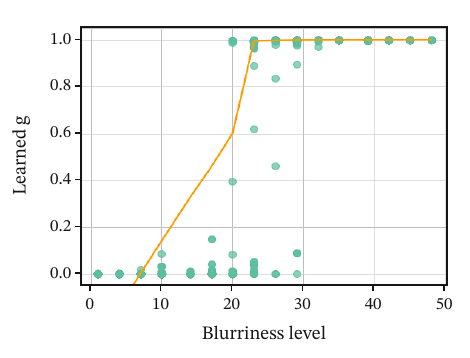}
        \caption{Training set}
    \end{subfigure}
    \hfill
    \begin{subfigure}[b]{0.45\textwidth}
        \centering
        \includegraphics[width=\textwidth]{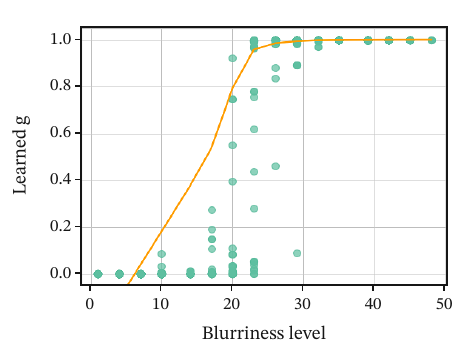}
        \caption{Test set}
    \end{subfigure}
    \caption{The learned feature score is positively associated with cell-image blur in both samples. Scatter plots of $\hat{g}$ vs. strengths of blurriness in the training and test sets. As expected, larger blurriness levels tend to correspond with larger values of $\hat{g}$.}
\label{fig:scatter-image-cell}
\end{figure}

\Cref{fig:nudging-cell} shows representative nudging results. Nudging iteratively adjusts the image-style representation along the gradient of~$\hat g$ while holding content fixed~\citep{wibisono2026causal}. It increases image blur while preserving the number of cells, consistent with the behavior captured by the learned feature-scoring function.

\begin{figure}[H]
    \centering
    \begin{subfigure}[b]{0.47\textwidth}
        \centering
        \includegraphics[width=\textwidth]{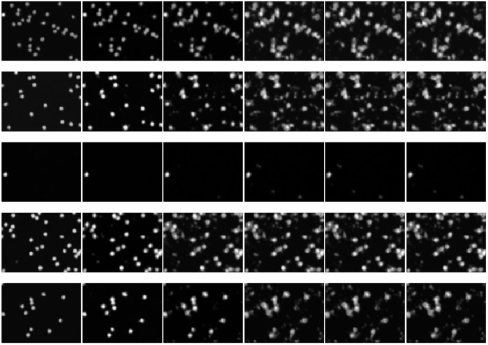}
    \end{subfigure}
    \hfill
    \begin{subfigure}[b]{0.47\textwidth}
        \centering
        \includegraphics[width=\textwidth]{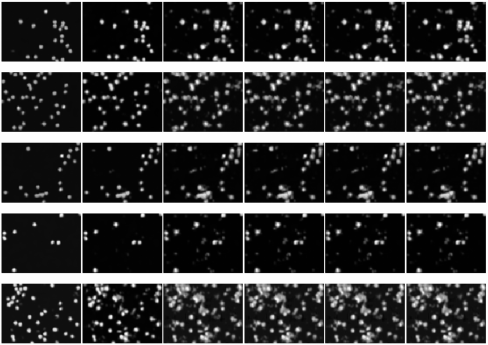}
    \end{subfigure}
    \caption{Nudging increases blur while holding cell-count content fixed. Nudging results for selected cell images. We observe that nudging generally makes images more blurry.}
\label{fig:nudging-cell}
\end{figure}

\FloatBarrier
\subsection{Experiment III: Empirical studies on simultaneous text-based treatments and outcomes}
\label{sec:exp-text-treatment-outcome}

The third experiment considers settings in which both treatment and outcome are unstructured. The treatment~$A$ is either a numerical vector or a text prompt, and the outcome~$Y$ is either a numerical vector or generated text. The goal is to learn paired feature-scoring functions~$f(A)$ and~$g(Y)$ that isolate the treatment-side and outcome-side directions most strongly associated with each other after adjustment for context, as developed in \Cref{sec:extension-hd-treatment-outcome}. The first study uses a controlled synthetic design with a known ground-truth relationship. The second studies LLM-assisted headline generation, in which prompt formality induces headline formality.

\subsubsection{Stylized treatment-plan example}
\label{sec:toy-example}

\textbf{Causal question.}
Consider a simplified study of an unstructured treatment plan and an unstructured follow-up measurement. The treatment vector~$A$ records three plan features, and the outcome vector~$Y$ records three post-treatment image or response features. The analyst wants to know which plan feature is linked to which outcome feature while ignoring irrelevant variation. We use a synthetic design so that the ground truth is known. The causal question is: \textit{which treatment coordinate is linked to which outcome coordinate?}

\textbf{Experimental setup.}
Only the first treatment coordinate affects the outcome: larger values of~$A_1$ produce larger values of~$Y_1$. The remaining coordinates~$A_2$, $A_3$, $Y_2$, and~$Y_3$ are noise. Therefore, the desired feature-scoring functions should depend on~$A_1$ and~$Y_1$ but ignore the other components. We estimate homogeneous functions~$\hat f(A)$ and~$\hat g(Y)$ using the MIF--MCF algorithm in \Cref{sec:extension-hd-treatment-outcome}. The complete distributions are in \Cref{app:emp-toy}.

\textbf{Results.}
\Cref{fig:toy} shows the learned scores on the test set. Panel~(a) shows that~$\hat f$ changes sharply with~$A_1$ but has no systematic relationship with~$A_2$ or~$A_3$. Panel~(b) shows that~$\hat g$ changes sharply with~$Y_1$ but not with~$Y_2$ or~$Y_3$. The paired feature-scoring functions therefore recover the correct treatment--outcome relationship and filter out irrelevant dimensions.

\begin{figure}[t]
    \centering
    \includegraphics[width=\textwidth]{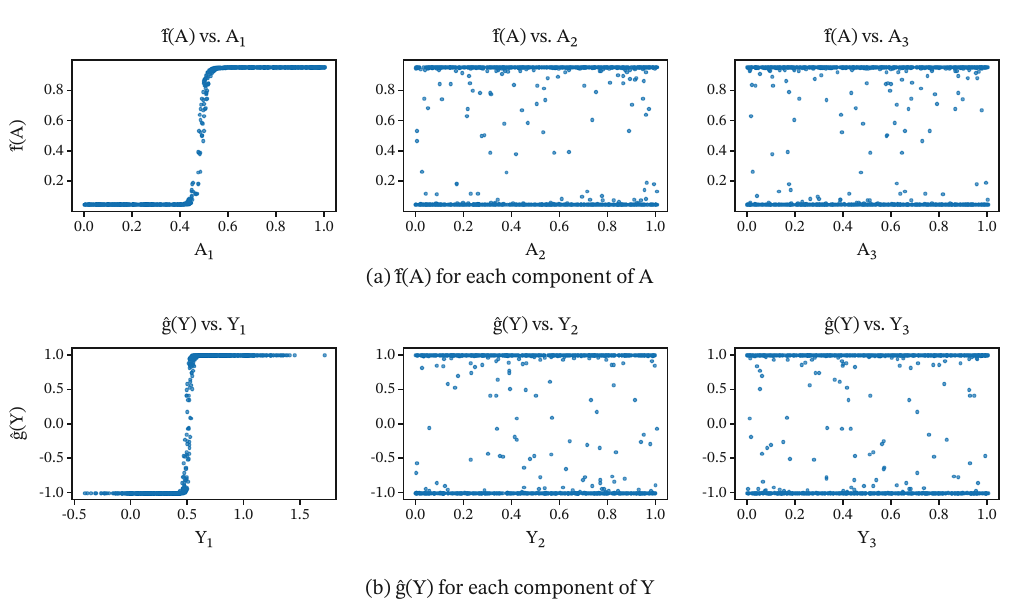}
    \caption{The paired feature-scoring functions ignore irrelevant treatment and outcome coordinates. Learned treatment and outcome features in the toy unstructured treatment-outcome experiment. Panel (a) shows $\hat f(A)$ against each component of $A$, and Panel (b) shows $\hat g(Y)$ against each component of $Y$. The learned features correctly identify that only $A_1$ and $Y_1$ are informative.}
\label{fig:toy}
\end{figure}

\subsubsection{Formality-inducing prompts in news headline generation}
\label{sec:headline-generation}

\textbf{Causal question.}
We next consider a more realistic text-to-text setting motivated by LLM-assisted content generation. Suppose a platform uses an instruction-tuned language model to generate short news headlines. The platform varies the prompt style, for example by requesting either a professional or a relaxed headline, and wants to understand how this prompt-level intervention changes the generated text. The treatment~$A$ is the text prompt, the outcome~$Y$ is the generated headline, and the context~$X$ is the news topic. The causal question is: \textit{which feature of the prompt induces which feature of the generated headline after accounting for topic?}

\textbf{Experimental setup.}
The binary topic distinguishes entertainment from politics and affects the probability that a prompt requests a formal headline. A formal prompt asks for a concise professional headline, while an informal prompt asks for a compact relaxed headline. Qwen2.5-1.5B-Instruct~\citep{qwen2024qwen25} generates the outcome after receiving the prompt and a separate topic instruction. This design creates confounding through topic: topic affects both the probability of receiving a formal prompt and the content of the generated headline. An algorithm that ignores topic may therefore confuse topic-related variation with prompt-induced style. The desired functions should instead recover prompt-side formality and outcome-side headline formality while adjusting for topic. Prompts and headlines are represented with the \textit{all-MiniLM-L6-v2} sentence-embedding model,\footnote{\url{https://huggingface.co/sentence-transformers/all-MiniLM-L6-v2}} implemented with Sentence Transformers~\citep{reimers2019sentence}. We estimate homogeneous feature-scoring functions~$\hat f(A)$ and~$\hat g(Y)$ using the MIF--MCF algorithm in \Cref{sec:extension-hd-treatment-outcome}. The complete four-step generation mechanism is in \Cref{app:emp-headline}.

\textbf{Results.}
\Cref{fig:avg_f_g} shows the expected formality pattern in both topics. Formal prompts receive high average~$\hat f(A)$ scores, while informal prompts receive low scores. The generated headlines follow the same pattern: headlines produced from formal prompts receive high average~$\hat g(Y)$ scores, while those produced from informal prompts receive low scores. Thus, the two learned functions recover the intended prompt-to-headline relationship. They identify prompt formality and the corresponding headline formality rather than merely learning the topic of the headline.

\begin{figure}[H]
    \centering
    \includegraphics[width=0.9\textwidth]{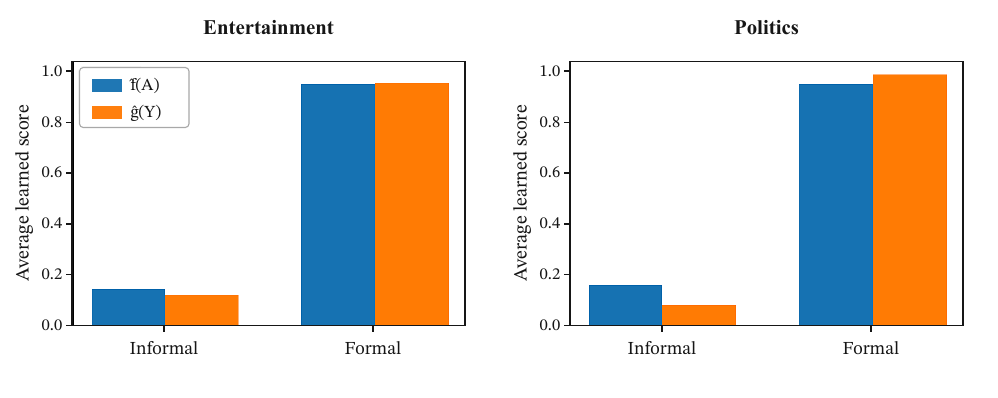}
    \caption{The paired feature-scoring functions recover prompt-to-headline formality after topic adjustment. Average learned $\hat f(A)$ and $\hat g(Y)$ values for informal and formal prompts in the headline generation experiment, stratified by topic. Formal prompts receive high $\hat f(A)$ values, and the corresponding generated headlines receive high $\hat g(Y)$ values in both entertainment and politics settings.}
    \label{fig:avg_f_g}
\end{figure}

\FloatBarrier

\section{Discussion}
\label{sec:discussion}

This paper develops algorithms for causal inference with unstructured outcomes such as text and images. In these settings, the average treatment effect has no immediate object: one cannot naturally subtract one clinical note, survey response, or image from another. We introduce the \emph{maximally contrasting feature (MCF)}, a learned feature of the outcome that reveals the strongest causal contrast between treated and control potential outcomes. We develop efficient estimators for MCF. We further extend MCF to heterogeneous effects and to settings in which both treatments and outcomes are unstructured. Across empirical studies, we find that MCF can recover treatment-relevant features in text and image outcomes; it also consistently improves over baseline summaries such as topic models. These results suggest MCF provides a general approach to studying interventions whose effects appear in the structure and content of rich outcomes.

\section*{Acknowledgements}

This work was supported in part by funding from the Office of Naval Research under grant N00014-23-1-2590, the National Science Foundation under grant No. 2310831, No. 2428059, No. 2435696, No. 2440954, a Michigan Institute for Data Science Propelling Original Data Science (PODS) grant, Two Sigma Investments LP, and  LG Management Development Institute AI Research. Any opinions, findings, and conclusions or recommendations expressed in this material are those of the authors and do not necessarily reflect the views of the sponsors.
\clearpage
\bibliographystyle{abbrvnat}
\bibliography{bibfile}
\clearpage

\appendix
\crefalias{section}{appendix}
\crefalias{subsection}{appendix}
\crefalias{subsubsection}{appendix}

\begin{center}
    \Large{\textbf{Supplementary Materials: Causal Inference with Unstructured Outcomes}}
\end{center}

\proofsection{prop:mcf-identification}{Identification of feature-specific contrasts}{app:pf-mcf-identification}
\begin{proof}
By SUTVA and unconfoundedness,
\[
\E\{g(Y(a))\}
=
\E\left[\E\{g(Y(a))\mid X\}\right]
=
\E\left[\E\{g(Y)\mid A=a,X\}\right].
\]
Subtracting the control mean from the treated mean gives \Cref{eq:mcf-adjustment}. The inverse-probability expression follows by iterated expectation. For example,
\[
\E\left\{\frac{A g(Y)}{e_0(X)}\mid X\right\}
=
\E\{g(Y)\mid A=1,X\}.
\]
The control term is analogous.
\end{proof}

\proofsection{prop:mcf-identification-uniqueness}{Identification and uniqueness of the MCF}{app:pf-mcf-identification-uniqueness}
\begin{proof}
The first statement follows from \Cref{prop:mcf-identification}. If the identified map \(g\mapsto\theta(g)\) has a unique maximizer, that maximizer is a functional of the observed-data law.

For the oracle maximizer,
\[
    \theta(g)
    =
    \int g(y)\{p_1(y)-p_0(y)\}\,\mathrm{d} y
    =
    \int g(y)\Delta(y)\,\mathrm{d}y.
\]
At each \(y\), the integrand is linear in \(g(y)\). It is maximized by taking \(g(y)=1\) when \(\Delta(y)>0\) and \(g(y)=0\) when \(\Delta(y)<0\). When \(\Delta(y)=0\), every value in \([0,1]\) gives the same contribution. If the tie set has reference probability zero, these arbitrary choices do not change the feature almost surely.
\end{proof}

\proofsection{prop:ext-2}{Characterization of the heterogeneous oracle MCF}{app:pf-ext-2}
\begin{proof}
For every feasible \(g\),
\[
T(g)
=
\int g(y,x)\,p(x)\,\Delta(y\mid x)\,dy\,dx.
\]
The integrand is pointwise linear in \(g(y,x)\). Since \(p(x)\geq0\), it is maximized by setting \(g(y,x)=1\) wherever \(\Delta(y\mid x)>0\) and \(g(y,x)=0\) wherever \(\Delta(y\mid x)<0\). On the tie set \(\Delta(y\mid x)=0\), every value in \([0,1]\) gives the same contribution. Consequently, every maximizer has the stated form away from the tie set, and a tie set of reference probability zero gives uniqueness almost surely.
\end{proof}

\proofsection{prop:cons}{Consistency}{app:pf-cons}
\begin{proof}
Let the observed-data vector be~\(Z=(X,A,Y)\), and define the population summand
\[
q(Z;\beta)=\frac{A g(Y;\beta)}{e_0(X)}-\frac{(1-A)g(Y;\beta)}{1-e_0(X)}.
\]
Define the oracle empirical criterion that uses the true propensity score as
\[
Q_N(\beta)=\frac{1}{N}\sum_{i=1}^N\left\{\frac{A_i g(Y_i;\beta)}{e_0(X_i)}-\frac{(1-A_i)g(Y_i;\beta)}{1-e_0(X_i)}\right\},
\]
and define the feasible empirical criterion as
\[
\widetilde Q_N(\beta)=\frac{1}{N}\sum_{i=1}^N\left\{\frac{A_i g(Y_i;\beta)}{\hat e(X_i)}-\frac{(1-A_i)g(Y_i;\beta)}{1-\hat e(X_i)}\right\}.
\]
The Glivenko--Cantelli condition in \Cref{prop:cons} gives
\[
\sup_{\beta\in\Theta}\left|Q_N(\beta)-\mathbb{E}\{q(Z;\beta)\}\right|=o_p(1).
\]

Let~\(\eta>0\) be the overlap margin in \Cref{asm:mcf-overlap}. Uniform consistency of~\(\hat e\) implies that, with probability tending to one, both~\(\hat e(X_i)\) and~\(1-\hat e(X_i)\) are at least~\(\eta/2\) for every observation. Because the feature-scoring functions satisfy~\(0\leq g\leq1\),
\[
\sup_{\beta\in\Theta}\left|\widetilde Q_N(\beta)-Q_N(\beta)\right|\leq\frac{2}{\eta^2}\sup_x|\hat e(x)-e_0(x)|=o_p(1).
\]
Therefore,
\[
\sup_{\beta\in\Theta}\left|\widetilde Q_N(\beta)-\mathbb{E}\{q(Z;\beta)\}\right|=o_p(1).
\]
Compactness of~\(\Theta\), uniqueness of the population maximizer~\(\beta_0\), and the argmax theorem now imply~\(\hat\beta\xrightarrow{p}\beta_0\), as claimed.
\end{proof}

\proofsection{prop:eff}{Asymptotic normality and efficiency}{app:pf-eff}
\begin{proof}
Write \(\phi_{\beta,e}(Z)=\phi(Z;\beta,e)\) for the IPW moment function in \Cref{eq:ipw-moment}. Decompose the estimating equation into the fixed-parameter estimation term
\[
W_{1N}=\mathbb{P}_N\phi_{\beta_0,\hat e}-\mathbb{P}\phi_{\beta_0,e_0},
\]
the empirical-process remainder
\[
W_{2N}=(\mathbb{P}_N-\mathbb{P})(\phi_{\hat\beta,\hat e}-\phi_{\beta_0,\hat e}),
\]
and the population linearization term
\[
W_{3N}=\mathbb{P}(\phi_{\hat\beta,\hat e}-\phi_{\beta_0,\hat e}).
\]
The population and empirical moment conditions imply
\[
0=W_{1N}+W_{2N}+W_{3N}.
\]

By the fixed-parameter efficiency expansion in \Cref{eq:tau-expansion},
\[
W_{1N}=(\mathbb{P}_N-\mathbb{P})\{\gamma(Z;\beta_0,e_0)\}+o_p(N^{-1/2}),
\]
where the efficient influence representation~\(\gamma\) is defined in \Cref{eq:efficient-representation}. By the Donsker condition, the consistency result in \Cref{prop:cons}, and the sample-splitting or cross-fitting condition, stochastic equicontinuity gives
\[
W_{2N}=o_p(N^{-1/2});
\]
see Lemma 19.24 of \citet{van2000asymptotic}.

Uniform differentiability of the population moment and convergence of the derivative matrix give
\[
W_{3N}=V(\beta_0,e_0)(\hat\beta-\beta_0)+o_p(\lVert\hat\beta-\beta_0\rVert).
\]
Combining the three terms yields
\[
0=(\mathbb{P}_N-\mathbb{P})\{\gamma(Z;\beta_0,e_0)\}+V(\beta_0,e_0)(\hat\beta-\beta_0)+o_p(N^{-1/2})+o_p(\lVert\hat\beta-\beta_0\rVert).
\]
Because \(V(\beta_0,e_0)\) is non-singular,
\[
\hat\beta-\beta_0=-V(\beta_0,e_0)^{-1}(\mathbb{P}_N-\mathbb{P})\{\gamma(Z;\beta_0,e_0)\}+o_p(N^{-1/2})+o_p(\lVert\hat\beta-\beta_0\rVert).
\]
Taking norms first gives \(\lVert\hat\beta-\beta_0\rVert=O_p(N^{-1/2})\). Substituting this rate into the preceding expansion gives
\[
\sqrt N(\hat\beta-\beta_0)=-V(\beta_0,e_0)^{-1}\sqrt N(\mathbb{P}_N-\mathbb{P})\{\gamma(Z;\beta_0,e_0)\}+o_p(1).
\]
Thus, the efficient influence function for~\(\beta_0\) is \(-V(\beta_0,e_0)^{-1}\gamma(Z;\beta_0,e_0)\), which proves asymptotic normality and semiparametric efficiency.
\end{proof}

\proofsection{prop:ext-3}{Budgeted oracle MCF}{app:pf-ext-3}
\begin{proof}
Let
\[
s(y)=\frac{\Delta(y)}{p_{\mathrm{ref}}(y)}.
\]
Then
\[
\theta(g)
=
\int g(y)s(y)p_{\mathrm{ref}}(y)\,dy.
\]
Let \(\lambda_\tau\) be an upper \(\tau\)-threshold of \(s(Y)\) under \(P_{\mathrm{ref}}\). The no-atom condition at this threshold implies
\[
P_{\mathrm{ref}}\{s(Y)\ge\lambda_\tau\}=\tau,
\]
so \(g^\star_\tau(y)=\mathbf{1}\{s(y)\ge\lambda_\tau\}\) satisfies the budget. For any other feasible \(g\),
\[
\begin{aligned}
\theta(g^\star_\tau)-\theta(g)
&=
\int \{g^\star_\tau(y)-g(y)\}
      \{s(y)-\lambda_\tau\}
      p_{\mathrm{ref}}(y)\,dy
\\
&\ge 0.
\end{aligned}
\]
The inequality holds pointwise: above the threshold, \(g^\star_\tau=1\) and \(s-\lambda_\tau\ge0\); below it, \(g^\star_\tau=0\) and \(s-\lambda_\tau<0\). The no-atom condition also gives uniqueness \(P_{\mathrm{ref}}\)-almost surely.

For \(0<\tau_1<\tau_2<\tau^\star\), both thresholds are positive and the selected region for \(\tau_2\) adds \(P_{\mathrm{ref}}\)-mass only where \(s>0\). Therefore,
\[
\theta(g^\star_{\tau_2})-\theta(g^\star_{\tau_1})
=
\int \{g^\star_{\tau_2}(y)-g^\star_{\tau_1}(y)\}
      s(y)p_{\mathrm{ref}}(y)\,dy
>0,
\]
which proves the stated monotonicity on \((0,\tau^\star)\).

It remains to verify the saturation statement. Let \(g_+(y)=\mathbf{1}\{s(y)>0\}\). At \(\tau=\tau^\star\), any feasible \(g\) satisfies
\[
\theta(g_+)-\theta(g)
=
\int_{\{s>0\}}\{1-g(y)\}s(y)\,dP_{\mathrm{ref}}(y)
+
\int_{\{s<0\}}g(y)\{-s(y)\}\,dP_{\mathrm{ref}}(y)
\ge0.
\]
Equality requires \(g=1\) on \(\{s>0\}\) and \(g=0\) on \(\{s<0\}\); the exact budget then requires \(g=0\) on \(\{s=0\}\), all \(P_{\mathrm{ref}}\)-almost surely. For \(\tau\in(\tau^\star,\tau^\star+\tau_0]\), assigning the additional mass to \(\{s=0\}\) attains the same value as \(g_+\). When \(\tau>\tau^\star+\tau_0\), every feasible rule must assign positive mass to \(\{s<0\}\), so its value is strictly below the saturated value. Finally, for \(\tau^\star+\tau_0<\tau_1<\tau_2\), the ranking argument implies that an optimizer at \(\tau_2\) selects all nonnegative-score mass and \(\tau_2-\tau^\star-\tau_0\) negative-score mass. Scaling its selection on \(\{s<0\}\) by \((\tau_1-\tau^\star-\tau_0)/(\tau_2-\tau^\star-\tau_0)\) gives a feasible rule at \(\tau_1\) with strictly larger value. Thus the exact-budget optimum is strictly decreasing beyond the plateau.
\end{proof}

\proofsection{prop:ext-4}{Heterogeneous budgeted oracle MCF}{app:pf-ext-4}
\begin{proof}
Write
\[
r(y,x)=\frac{p(x)\Delta(y\mid x)}{p_{\mathrm{ref}}(y,x)}.
\]
Then
\[
\int g(y,x)p(x)\Delta(y\mid x)\,dy\,dx
=
\int g(y,x)r(y,x)p_{\mathrm{ref}}(y,x)\,dy\,dx.
\]
Thus the problem is to choose a function \(0\le g\le1\) with \(P_{\mathrm{ref}}\)-mass \(\tau\) that maximizes the average of \(r(Y,X)\) over the selected region. The optimal rule selects the largest values of \(r\). Let \(\lambda_\tau\) be an upper \(\tau\)-threshold. Under the no-atom condition at this threshold, the optimizer is unique \(P_{\mathrm{ref}}\)-almost surely and is given by
\[
g^\star_{X,\tau}(y,x)
=
\mathbf{1}\{r(y,x)\ge \lambda_\tau\}.
\]
For \(\tau<\tau^\star_X\), the threshold is positive, so lowering it adds only pairs with positive conditional density difference. Hence \(\theta_X(g^\star_{X,\tau})\) is increasing in \(\tau\).

The saturation claims follow from the same joint-space ordering. Let
\[
g_{X,+}(y,x)=\mathbf{1}\{r(y,x)>0\}.
\]
Because \(p(x)>0\) and \(p_{\mathrm{ref}}(y,x)>0\) on the relevant support, the events \(r(y,x)>0\), \(r(y,x)=0\), and \(r(y,x)<0\) coincide with the corresponding signs of \(\Delta(y\mid x)\). Thus, at \(\tau=\tau^\star_X\), every feasible \(g\) satisfies
\[
\theta_X(g_{X,+})-\theta_X(g)
=
\int_{\{r>0\}}\{1-g(y,x)\}r(y,x)\,dP_{\mathrm{ref}}(y,x)
+
\int_{\{r<0\}}g(y,x)\{-r(y,x)\}\,dP_{\mathrm{ref}}(y,x)
\geq0.
\]
Equality requires selection of the positive-score set and exclusion of the negative-score set, up to \(P_{\mathrm{ref}}\)-null sets. For \(\tau\in(\tau^\star_X,\tau^\star_X+\tau_{0,X}]\), the additional exact-budget mass can be assigned to \(\{r=0\}\) without changing the objective. When \(\tau>\tau^\star_X+\tau_{0,X}\), every exact-budget rule must assign positive mass to \(\{r<0\}\). Under an upper-bound budget, \(g_{X,+}\) remains feasible for every bound at least \(\tau^\star_X\) and pointwise maximizes the integrand, so the optimum saturates at \(\tau^\star_X\).
\end{proof}

\proofsection{prop:hdto-equivalent}{Equivalent forms of the joint feature objective}{app:proof_hdto_equivalent}
\begin{proof}
First,
\[
    \E\!
    \left[f(A)\E\{g(Y)\mid X\}\right]
    =
    \E\!
    \left[\E\{f(A)\mid X\}\E\{g(Y)\mid X\}\right].
\]
It follows that
\[
\begin{aligned}
&\E\!
\left[f(A)\{g(Y)-\E(g(Y)\mid X)\}\right]
\\
&\qquad=
\E\{f(A)g(Y)\}
-
\E\!
\left[\E\{f(A)\mid X\}\E\{g(Y)\mid X\}\right].
\end{aligned}
\]
Expanding the centered product gives the same expression:
\[
\begin{aligned}
&\E\!
\left[
    \{f(A)-\E(f(A)\mid X)\}
    \{g(Y)-\E(g(Y)\mid X)\}
\right]
\\
&\qquad=
\E\{f(A)g(Y)\}
-
\E\!
\left[\E\{f(A)\mid X\}\E\{g(Y)\mid X\}\right].
\end{aligned}
\]
The final equality follows from the definition of conditional covariance and iterated expectation.
\end{proof}

\proofsection{prop:negative-outcome-hdto}{Negative control outcomes recover the centered objective}{app:proof_negative-outcome-hdto}
\begin{proof}
Since \(Y^\dagger\perp A\mid X\),
\[
    \E\{f(A)g(Y^\dagger)\mid X\}
    =
    \E\{f(A)\mid X\}\E\{g(Y^\dagger)\mid X\}.
\]
Since \(Y^\dagger\mid X\) and \(Y\mid X\) have the same distribution,
\[
    \E\{g(Y^\dagger)\mid X\}
    =
    \E\{g(Y)\mid X\}.
\]
Therefore
\[
    \E\{f(A)g(Y^\dagger)\}
    =
    \E\!
    \left[\E\{f(A)\mid X\}\E\{g(Y)\mid X\}\right].
\]
Also,
\[
    \E\!
    \left[f(A)\E\{g(Y)\mid X\}\right]
    =
    \E\!
    \left[\E\{f(A)\mid X\}\E\{g(Y)\mid X\}\right].
\]
Combining the two equalities proves the claim.
\end{proof}

\section{Why inverse propensity weighting rather than plug-in estimation?}
\label{sec:why-ipw}

The MCF is defined as a maximizer over feature-scoring functions. In estimation, however, we do not optimize over all functions directly. We work with a parameterized class
\[
    \mathcal{G}_{\Theta}
    =
    \{g(\cdot;\beta):\beta\in\Theta\},
\]
where \(\beta\in\Theta\subseteq\mathbb{R}^{n_\beta}\) denotes the trainable parameter vector of the feature-scoring function and \(n_\beta\) is its dimension. The population target in this class is
\begin{equation}
\label{eq:beta-target-alt}
    \beta_0
    \in
    \operatorname*{arg\,max}_{\beta\in\Theta}
    \theta(\beta),
    \qquad
    \theta(\beta)
    =
    \mathbb{E}\left[g\{Y(1);\beta\}-g\{Y(0);\beta\}\right].
\end{equation}
Thus \(g(Y;\beta_0)\) is the feature in the chosen parameterized class whose average potential value changes most when treatment is set from \(0\) to \(1\).

By the identification result in \Cref{prop:mcf-identification}, for every fixed \(\beta\),
\begin{equation}
\label{eq:beta-adjustment}
    \theta(\beta)
    =
    \mathbb{E}\left[
        \mu_{1,\beta}(X)-\mu_{0,\beta}(X)
    \right],
\end{equation}
where
\[
    \mu_{a,\beta}(x)
    =
    \mathbb{E}\{g(Y;\beta)\mid X=x,A=a\},
    \qquad a\in\{0,1\}.
\]
For a fixed value of \(\beta\), this is the usual regression-adjustment formula: estimate the treated and control conditional means of the scalar outcome \(g(Y;\beta)\), average their difference over the covariate distribution, and obtain the feature-specific causal contrast.

The difficulty is that \(\beta\) is the object being learned. During training, each update of \(\beta\) changes the scalar outcome \(g(Y;\beta)\). Consequently, the regression target~\(\mathbb{E}\{g(Y;\beta)\mid X=x,A=a\}\) also changes as the feature-scoring function changes. A plug-in implementation would therefore require learning conditional mean functions for a moving outcome. In text and image applications, this creates a nested learning problem: the outer optimization updates the feature-scoring function~\(g(\cdot;\beta)\), while the inner optimization updates the nuisance regressions for the current scalar outcome~\(g(Y;\beta)\).

The same issue appears if the conditional mean is written through the conditional density of the represented outcome. Let \(\eta_a(y\mid x)\) denote the density of \(Y\) given \(X=x\) and \(A=a\). Then
\[
    \mu_{a,\beta}(x)
    =
    \int g(y;\beta)\eta_a(y\mid x)\,dy.
\]
Thus the plug-in route either requires estimating the conditional density \(\eta_a(y\mid x)\) of the unstructured outcome, or repeatedly fitting regressions of the changing scalar outcome \(g(Y;\beta)\) on \(X\) within each treatment group. Both are burdensome when \(Y\) is a text or image representation and \(g(\cdot;\beta)\) is trained by gradient-based optimization.

\textbf{The plug-in first-order condition.}
The plug-in difficulty can also be seen from the first-order condition for \(\beta_0\). Suppose \(\beta_0\) is an interior maximizer of \(\theta(\beta)\), and suppose \(g(y;\beta)\) is differentiable with respect to \(\beta\). Define
\[
    \ell_\beta(y)
    =
    \frac{\partial g(y;\beta)}{\partial \beta}.
\]
Let \(p_a(y)\) denote the density of the potential outcome \(Y(a)\). Under the identification assumptions, this density can be written as
\[
    p_a(y)
    =
    \int \eta_a(y\mid x)\,d\mathbb{P}_X(x),
\]
where \(\mathbb{P}_X\) is the marginal distribution of \(X\). Since
\[
    \theta(\beta)
    =
    \int g(y;\beta)\{p_1(y)-p_0(y)\}\,dy,
\]
differentiating with respect to \(\beta\) gives the moment function
\begin{equation}
\label{eq:moment-beta}
    m(\beta)
    =
    \frac{\partial \theta(\beta)}{\partial \beta}
    =
    \int \ell_\beta(y)\{p_1(y)-p_0(y)\}\,dy.
\end{equation}
Therefore an interior maximizer \(\beta_0\) satisfies
\[
    m(\beta_0)=0.
\]

Equivalently, by adjustment,
\[
    m(\beta)
    =
    \mathbb{E}\left[
        r_{1,\beta}(X)-r_{0,\beta}(X)
    \right],
\]
where
\[
    r_{a,\beta}(x)
    =
    \mathbb{E}\{\ell_\beta(Y)\mid X=x,A=a\},
    \qquad a\in\{0,1\}.
\]
A direct plug-in estimator of the first-order condition would be
\begin{equation}
\label{eq:plugin-moment}
    \widehat m_{\mathrm{PI}}(\beta)
    =
    \frac{1}{N}
    \sum_{i=1}^N
    \left\{
        \widehat r_{1,\beta}(X_i)
        -
        \widehat r_{0,\beta}(X_i)
    \right\},
\end{equation}
where \(\widehat r_{a,\beta}(x)\) estimates
\(\mathbb{E}\{\ell_\beta(Y)\mid X=x,A=a\}\). This moment is natural, but it inherits the same moving-target problem: the regression target \(\ell_\beta(Y)\) changes whenever \(\beta\) changes.

\begin{proposition}[MCF first-order condition]
\label{prop:mcf-first-order}
Assume that \(\beta_0\) is an interior maximizer of \(\theta(\beta)\), that \(\theta(\beta)\) is differentiable in a neighborhood of \(\beta_0\), and that differentiation may be interchanged with integration. Then \(\beta_0\) satisfies the moment condition \(m(\beta_0)=0\), where \(m(\beta)\) is defined in \Cref{eq:moment-beta}.
\end{proposition}
The proof is in \Cref{app:pf-moment-alt}.

\textbf{Bias correction and its nuisance functions.}
For a generic scalar variable \(T\), define the adjusted mean under treatment level \(a\) as
\[
    \psi_a(T)
    =
    \mathbb{E}\left[
        \mathbb{E}(T\mid X,A=a)
    \right].
\]
Let
\[
    e_a(X)=\mathbb{P}(A=a\mid X).
\]
The efficient influence function for \(\psi_a(T)\) is
\[
    \phi_a(T;\mathbb{P})
    =
    \frac{\mathbf{1}(A=a)}{e_a(X)}
    \left\{
        T-\mathbb{E}(T\mid X,A=a)
    \right\}
    +
    \mathbb{E}(T\mid X,A=a)
    -
    \psi_a(T).
\]
This is the standard efficient influence function for an adjusted mean~\citep{kennedy2023semiparametric,kennedy2022semiparametric}. Applying the formula to \(T=\ell_\beta(Y)\), the efficient influence function for \(m(\beta)\) is
\[
    \phi_1\{\ell_\beta(Y);\mathbb{P}\}
    -
    \phi_0\{\ell_\beta(Y);\mathbb{P}\}.
\]
Thus a bias-corrected moment estimator is
\begin{equation}
\label{eq:bc-moment}
    \widehat m_{\mathrm{BC}}(\beta)
    =
    \widehat m_{\mathrm{PI}}(\beta)
    +
    \frac{1}{N}
    \sum_{i=1}^N
    \left[
        \phi_1\{\ell_\beta(Y_i);\widehat{\mathbb{P}}\}
        -
        \phi_0\{\ell_\beta(Y_i);\widehat{\mathbb{P}}\}
    \right],
\end{equation}
where \(\widehat{\mathbb{P}}\) means that the unknown propensity scores and conditional expectations are replaced by estimates. Equivalently, define the estimated propensity score as
\[
    \widehat e(x)=\widehat{\mathbb{P}}(A=1\mid X=x).
\]
Then \Cref{eq:bc-moment} can be written as
\[
    \widehat m_{\mathrm{BC}}(\beta)
    =
    \frac{1}{N}
    \sum_{i=1}^N
    \Bigg[
        \frac{A_i}{\widehat e(X_i)}
        \{\ell_\beta(Y_i)-\widehat r_{1,\beta}(X_i)\}
        -
        \frac{1-A_i}{1-\widehat e(X_i)}
        \{\ell_\beta(Y_i)-\widehat r_{0,\beta}(X_i)\}
        +
        \widehat r_{1,\beta}(X_i)
        -
        \widehat r_{0,\beta}(X_i)
    \Bigg].
\]
Bias-corrected plug-in estimators are statistically attractive when the nuisance functions can be estimated accurately. Under standard regularity conditions, including \(N^{-1/4}\)-type-rate conditions for the propensity estimator~\(\widehat e\) and the outcome-regression-gradient estimators~\(\widehat r_{a,\beta}\), solving a bias-corrected moment equation can yield an asymptotically normal and semiparametrically efficient estimator~\citep{kennedy2023semiparametric,kennedy2022semiparametric}. In the present problem, however, the bias-corrected moment still requires estimating \(\widehat r_{a,\beta}\) for a changing derivative outcome \(\ell_\beta(Y)\). The bias correction improves the statistical expansion once the nuisance functions are available; it does not remove the nested learning problem created by the fact that the nuisance regressions depend on \(\beta\).

One may also view these moments as gradients of plug-in objectives. Solving \(\widehat m_{\mathrm{PI}}(\beta)=0\) corresponds to finding a stationary point of
\[
    \frac{1}{N}
    \sum_{i=1}^N
    \left\{
        \widehat\mu_{1,\beta}(X_i)
        -
        \widehat\mu_{0,\beta}(X_i)
    \right\},
\]
where
\[
    \widehat\mu_{a,\beta}(x)
    \approx
    \mathbb{E}\{g(Y;\beta)\mid X=x,A=a\}.
\]
Similarly, solving \(\widehat m_{\mathrm{BC}}(\beta)=0\) corresponds to finding a stationary point of the augmented objective
\[
    \frac{1}{N}
    \sum_{i=1}^N
    \left\{
        \widehat\mu_{1,\beta}(X_i)
        -
        \widehat\mu_{0,\beta}(X_i)
    \right\}
    +
    \frac{1}{N}
    \sum_{i=1}^N
    \left[
        \phi_1\{g(Y_i;\beta);\widehat{\mathbb{P}}\}
        -
        \phi_0\{g(Y_i;\beta);\widehat{\mathbb{P}}\}
    \right].
\]
These objectives are statistically natural, but they require either estimating the conditional density \(\eta_a(y\mid x)\) of the unstructured outcome, or repeatedly updating the nuisance regressions \(\widehat\mu_{a,\beta}\) and \(\widehat r_{a,\beta}\) as \(\beta\) changes.

One possible way to avoid explicit density estimation is bi-level optimization. For each current value of \(\beta\), one learns the nuisance regressions
\[
    x\mapsto
    \mathbb{E}\{g(Y;\beta)\mid X=x,A=a\},
    \qquad a\in\{0,1\},
\]
and then substitutes these fitted regressions into the outer objective. \citet{melnychuk2023normalizing} use normalizing flows to estimate counterfactual outcome distributions. In our setting, however, fitting or updating such nuisance models as~\(\beta\) changes creates a nested optimization problem for unstructured or representation-valued outcomes. We implemented two bi-level approaches: one based on hypernetworks~\citep{lorraine2018stochastic} and one based on implicit differentiation~\citep{lorraine2020optimizing}. In our experiments, both were unstable and sensitive to initialization.

\proofsection{prop:mcf-first-order}{MCF first-order condition}{app:pf-moment-alt}
\begin{proof}
Using the density representation of the potential outcomes,
\[
    \theta(\beta)
    =
    \mathbb{E}\left[g\{Y(1);\beta\}-g\{Y(0);\beta\}\right]
    =
    \int g(y;\beta)\{p_1(y)-p_0(y)\}\,dy.
\]
Differentiating both sides with respect to \(\beta\) gives
\[
    \frac{\partial \theta(\beta)}{\partial \beta}
    =
    \int
    \frac{\partial g(y;\beta)}{\partial \beta}
    \{p_1(y)-p_0(y)\}\,dy
    =
    m(\beta).
\]
Since \(\beta_0\) is an interior maximizer and \(\theta(\beta)\) is differentiable at \(\beta_0\), its first-order condition is
\[
    \frac{\partial \theta(\beta)}{\partial \beta}\bigg|_{\beta=\beta_0}
    =
    0.
\]
Hence \(m(\beta_0)=0\).
\end{proof}

\section{What does the MCF actually learn? A Gaussian example}
\label{sec:mcf-learn}
We now provide a more concrete interpretation of the MCF in terms of interpretable text features. The main point is that the MCF should not be interpreted as separately identifying individual features of the text. Instead, it learns a single score that combines multiple features, and the selected texts are obtained by thresholding this score. We first make this point in the unbudgeted setting, and then explain how the budget changes the threshold but not the underlying direction.

Throughout this section, we focus on the case where $g$ only depends on $Y$, noting that an extension to the case where $g$ depends on both $Y$ and $X$ is straightforward.

\paragraph{Unbudgeted Gaussian MCF.} Suppose that the outcome text embedding can be summarized by two
interpretable features, \(Y=(F,P)^\top\), where \(F\) denotes formality and \(P\) denotes politeness. Let \(p_a(y)\) denote the density of
\(Y(a)\) for \(a\in\{0,1\}\). The population MCF contrast is
\[
T(g)
=
\mathbb{E}\left(g(Y(1))\right)
-
\mathbb{E}\left(g(Y(0))\right)
=
\int g(y)\left(p_1(y)-p_0(y)\right)\,dy,
\]
where $0\le g\le 1$. From \Cref{prop:mcf-identification-uniqueness}, the population maximizer is
$
g^*(y)=\mathds{1}\left(p_1(y)>p_0(y)\right)$. Thus, the MCF selects precisely those outcome texts whose interventional density is larger under
treatment than under control.

This selection rule becomes especially transparent under a Gaussian equal-covariance model. We assume the interventional distributions \(Y(a)\sim\mathcal N(\mu_a,\Sigma)\) for \(a\in\{0,1\}\),
where
\[
\mu_a =
\begin{pmatrix}
\mu_{F,a}\\
\mu_{P,a}
\end{pmatrix},
\qquad
\Sigma=
\begin{pmatrix}
\sigma_F^2 & \rho \sigma_F \sigma_P\\
\rho \sigma_F \sigma_P & \sigma_P^2
\end{pmatrix},
\qquad |\rho|<1,
\]
the vector~\(\mu_a\) is the intervention-specific mean, the common matrix~\(\Sigma\) is the covariance, the quantities~\(\sigma_F\) and~\(\sigma_P\) are the marginal standard deviations, and~\(\rho\) is their correlation. We define the interventional mean-shift vector as~\(\delta:=\mu_1-\mu_0=(\Delta_F,\Delta_P)^\top\), where~\(\Delta_F\) is the treatment-induced shift in formality and~\(\Delta_P\) is the treatment-induced shift in politeness. A special case is~\(\Delta_P=0\), when treatment affects only formality and not politeness. In that case, politeness is \emph{not causally shifted}, although it may still be correlated with formality through~\(\rho\neq0\).

Under the Gaussian model, we have the following result.
\begin{proposition}[Gaussian MCF]
\label{prop:gaussian-mcf}
Under the equal-covariance Gaussian model,
\[
g^*(y)=\mathds{1}(w^\top y > c),
\]
where \(w=\Sigma^{-1}(\mu_1-\mu_0)=\Sigma^{-1}\delta\) is the discriminant direction and \(c=\{\mu_1^\top\Sigma^{-1}\mu_1-\mu_0^\top\Sigma^{-1}\mu_0\}/2\) is its threshold.
\end{proposition}
The proof is in \Cref{app:pf-gauss-mcf}. From \Cref{prop:gaussian-mcf}, the MCF learns a threshold on a single linear score~$w^\top y$. Its formality and politeness coefficients are
\[
w_F =
\frac{\Delta_F/\sigma_F^2 - \rho \Delta_P/(\sigma_F\sigma_P)}
{1-\rho^2}, \qquad
w_P = \frac{\Delta_P/\sigma_P^2 - \rho \Delta_F/(\sigma_F\sigma_P)}
{1-\rho^2}.
\]
Even when both features are affected by the treatment, the learned MCF does not produce two separate directions; it produces one combined direction. That being said, we can show that the learned MCF remains aligned with the true interventional shift.
\begin{proposition}[Positive alignment]
\label{prop:pos-align}
Let $\delta=\mu_1-\mu_0$ and $w=\Sigma^{-1}\delta$. Then,
\[
\delta^\top w = \delta^\top \Sigma^{-1}\delta > 0
\qquad\text{whenever } \delta \neq 0.
\]
\end{proposition}
The proof is in \Cref{app:pf-pos-align}. The next proposition turns this inner-product identity into a direct statement about how the MCF score changes along the causal shift.
\begin{proposition}[Moving along the causal shift increases the MCF score]
\label{prop:causal-shift}
For any $y\in\mathbb{R}^2$ and any $t\in \mathbb{R}$, we have
\[
w^\top (y+t\delta) = w^\top y + t\,\delta^\top \Sigma^{-1}\delta.
\]
If $\delta\neq 0$, the MCF score is strictly increasing in $t$; if $\delta=0$, it is constant.
\end{proposition}
The proof is in \Cref{app:pf-causal-shift}. When $\delta\neq0$, \Cref{prop:pos-align,prop:causal-shift} show that although $w$ is covariance-adjusted, it still points in the correct overall interventional direction. In particular, moving outcomes in the true interventional mean-shift direction increases the MCF score.

\paragraph{A signal-to-noise perspective.} Under the same condition, a key property of the MCF is that it follows the linear direction that maximizes the signal-to-noise ratio defined below. For any nonzero candidate direction~$u\in\mathbb{R}^2$, consider the linear score~$S_u(Y):=u^\top Y$. We have
\[
S_u(Y(a))\sim \mathcal N(u^\top \mu_a,\; u^\top \Sigma u).
\]
In particular, its mean shift is $u^\top \delta$, and its variance is $u^\top \Sigma u$. Define the squared signal-to-noise ratio
\[
\mathrm{SNR}(u):=\frac{(u^\top \delta)^2}{u^\top \Sigma u}.
\]

When $\delta\neq0$, \Cref{prop:mcf-snr} shows that the MCF maximizes the SNR. In other words, among all linear scores, the MCF finds the direction with the largest mean separation relative to the variance.
\begin{proposition}[MCF maximizes linear SNR]
\label{prop:mcf-snr}
Assume $\delta\neq0$, and let $w=\Sigma^{-1}\delta$. Then, $\mathrm{SNR}(u)\le \delta^\top \Sigma^{-1}\delta$ for every nonzero $u$, with equality if and only if $u$ is proportional to $w$.
\end{proposition}
The proof is in \Cref{app:pf-mcf-snr}. The following remark isolates the special case in which the intervention induces no mean shift in politeness.
\begin{remark}
In the special case where $\Delta_P = 0$, so the intervention induces no mean shift in politeness, MCF yields an SNR of
\[
\frac{1}{1 - \rho^2} \cdot \frac{\Delta_F^2}{\sigma_F^2},
\]
compared to $\Delta_F^2/\sigma_F^2$ when we restrict the second coordinate of $u$ to be zero (i.e., using formality alone). Since $|\rho| < 1$, we have $1/(1 - \rho^2) > 1$, so the MCF achieves strictly higher SNR whenever $\rho \neq 0$ and $\Delta_F\neq0$. Moreover, the gain factor $1/(1 - \rho^2)$ is increasing in $|\rho|$, so the benefit of incorporating the second coordinate becomes larger as the correlation between formality and politeness increases. Thus, although the intervention induces no mean shift in politeness, the correlated politeness coordinate can increase the discriminative signal-to-noise ratio.
\end{remark}

\paragraph{Comparing the contributions of different features.} To compare the contributions of different text features, it is useful to work with scale-normalized
coefficients, especially when the scales of the features differ. Define the scale-normalized features as
\[
F^* := \frac{F}{\sigma_F}, \qquad P^* := \frac{P}{\sigma_P},
\]
whence $w_F F + w_P P
=
(w_F \sigma_F)F^* + (w_P \sigma_P)P^*$.
Define the scale-normalized coefficients
\[
\tilde w_F := w_F \sigma_F, \qquad \tilde w_P := w_P \sigma_P,
\]
and the standardized treatment effects
\[
z_F := \frac{\Delta_F}{\sigma_F},
\qquad
z_P := \frac{\Delta_P}{\sigma_P}.
\]
Using simple algebra, the scale-normalized MCF coefficients are
\[
\tilde w_F = \frac{z_F - \rho z_P}{1-\rho^2},
\qquad
\tilde w_P = \frac{z_P - \rho z_F}{1-\rho^2}.
\]
\Cref{prop:ordering} shows that, after scale normalization, the relative size of the coefficients reflects the relative size of the normalized
treatment effects. For example, if formality has a larger normalized treatment shift than politeness, then
the scale-normalized MCF coefficient for formality is larger in magnitude. Its proof follows from simple algebra.
\begin{proposition}[Ordering of scale-normalized coefficients]
\label{prop:ordering}
With
\[
\tilde w_F = \frac{z_F - \rho z_P}{1-\rho^2},
\qquad
\tilde w_P = \frac{z_P - \rho z_F}{1-\rho^2},
\]
we have $|\tilde w_F| > |\tilde w_P|
\iff
|z_F| > |z_P|$.
\end{proposition}
The proof is in \Cref{app:pf-ordering}. \textbf{Budgeted Gaussian MCF.}
We now return to the budgeted formulation. Suppose the budget is defined with respect to the pooled interventional reference density
\[
p(y)
:=
\pi p_1(y)+(1-\pi)p_0(y),
\qquad
\pi\in(0,1).
\]
Here, \(\pi\) is the mixture weight.
When $\mu_1\neq\mu_0$, the equal-covariance Gaussian model gives the budgeted MCF the same direction as the unbudgeted MCF; only the threshold changes.
\begin{proposition}[Budgeted Gaussian MCF]
\label{prop:budgeted-gaussian-mcf}
In the equal-covariance Gaussian model with $\mu_1\neq\mu_0$ and pooled reference distribution $p(y):=\pi p_1(y)+(1-\pi)p_0(y)$, the budgeted MCF takes the form
\[
g_\tau^*(y)=\mathds{1}\{w^\top y \ge t_\tau\},
\qquad
w=\Sigma^{-1}(\mu_1-\mu_0).
\]
Concretely, only the threshold $t_\tau$ changes with $\tau$. The direction $w$ is unchanged.
\end{proposition}
The proof is in \Cref{app:pf-budgeted-gaussian-mcf}. Note that the budget does not discover new features or new directions: it simply moves the threshold along the same MCF direction. Concretely, in the formality and politeness situation, the budgeted formulation \textit{does not first discover the strongest feature and then the weaker one}; instead, both features enter through the same combined score $w^\top y$ at all budgets, and increasing the budget reveals less extreme points along the same ranking.

Finally, we formalize \textit{how the composition of selected texts changes as the budget increases}. We focus on the probability of selecting texts that are more polite than formal across different score levels. Throughout, we assume that formality has a larger effect than politeness and work with features on a common unit-variance scale so that differences in scale do not affect our comparisons.

\begin{proposition}[Small budgets rarely select texts that are polite but not formal]
\label{prop:small-budget-selection}
Suppose that under the reference distribution used for the budget constraint,
\[
Y=\begin{pmatrix}F\\P\end{pmatrix}
\sim \mathcal N\!\left(
\begin{pmatrix}\mu_F\\ \mu_P\end{pmatrix},
\begin{pmatrix}
1 & \rho\\
\rho & 1
\end{pmatrix}
\right),
\qquad -1<\rho<1,
\]
and let the MCF score be $S=\alpha_F F+\alpha_P P$, where $\alpha_F>\alpha_P>0$. For $\kappa\ge 0$, define the event $E_\kappa:=\{P-F\ge \kappa\}$, which represents texts that are substantially more polite than formal. Then,
\begin{enumerate}
\item The conditional probability $q(t):=\Pr(E_\kappa\mid S\ge t)$ is decreasing in $t$.

\item More generally, for any $t_1>t_2>t_3$, $\Pr(E_\kappa\mid t_2\le S<t_1)
\;\le\;
\Pr(E_\kappa\mid t_3\le S<t_2)$. Thus, lower score bands contain a larger proportion of polite-but-not-formal texts.

\item If the budgeted MCF rule selects the top $\tau$ fraction according to $S$, i.e.,
\[
\mathcal S_\tau=\{S\ge t_\tau\},
\qquad
\Pr(S\ge t_\tau)=\tau,
\]
then $t_\tau$ decreases as $\tau$ increases, and therefore $\Pr(E_\kappa\mid\mathcal S_\tau)$ is increasing in $\tau$.
\end{enumerate}
\end{proposition}
The proof is in \Cref{app:pf-thm-budget}. \Cref{prop:small-budget-selection} formalizes the intuition that, when formality has the larger coefficient in the MCF score, very small budgets mainly select texts that are strong along the formality direction. As the budget increases and the threshold decreases, the newly added texts may contain a larger proportion of texts whose inclusion is supported by the secondary feature, here politeness.

\proofsection{prop:gaussian-mcf}{Gaussian MCF}{app:pf-gauss-mcf}
\begin{proof}
For equal-covariance Gaussians,
\[
\log\frac{p_1(y)}{p_0(y)}
=
-\frac12 (y-\mu_1)^\top \Sigma^{-1}(y-\mu_1)
+\frac12 (y-\mu_0)^\top \Sigma^{-1}(y-\mu_0).
\]
Expanding and canceling the quadratic terms gives
\[
\log\frac{p_1(y)}{p_0(y)}
=
(\mu_1-\mu_0)^\top \Sigma^{-1} y
-
\frac12\left(\mu_1^\top\Sigma^{-1}\mu_1-\mu_0^\top\Sigma^{-1}\mu_0\right).
\]
Hence
\[
p_1(y)>p_0(y)
\iff
w^\top y > c,
\]
with \(w\) and \(c\) as stated.
\end{proof}

\proofsection{prop:pos-align}{Positive alignment}{app:pf-pos-align}
\begin{proof}
Since \(\Sigma\) is positive definite, \(\Sigma^{-1}\) is also positive definite. Therefore, \(\delta^\top \Sigma^{-1}\delta>0\) for every nonzero \(\delta\).
\end{proof}

\proofsection{prop:causal-shift}{Moving along the causal shift increases the MCF score}{app:pf-causal-shift}
\begin{proof}
By linearity,
\[
w^\top(y+t\delta)
=
w^\top y+t w^\top\delta
=
w^\top y+t\,\delta^\top\Sigma^{-1}\delta.
\]
The coefficient of \(t\) is strictly positive when \(\delta\neq0\) and is zero when \(\delta=0\).
\end{proof}

\proofsection{prop:mcf-snr}{MCF maximizes linear SNR}{app:pf-mcf-snr}
\begin{proof}
Let \(v=\Sigma^{1/2}u\), so that \(u=\Sigma^{-1/2}v\). Then
\[
\mathrm{SNR}(u)
=
\frac{(u^\top \delta)^2}{u^\top \Sigma u}
=
\frac{(v^\top \Sigma^{-1/2}\delta)^2}{v^\top v}.
\]
By Cauchy--Schwarz,
\[
(v^\top \Sigma^{-1/2}\delta)^2
\le
(v^\top v)(\delta^\top \Sigma^{-1}\delta).
\]
Hence
\[
\mathrm{SNR}(u)\le \delta^\top \Sigma^{-1}\delta.
\]
Equality holds if and only if \(v\) is proportional to \(\Sigma^{-1/2}\delta\); that is, \(u\) is proportional to \(\Sigma^{-1}\delta=w\).
\end{proof}

\proofsection{prop:ordering}{Ordering of scale-normalized coefficients}{app:pf-ordering}
\begin{proof}
Using the definitions of the scale-normalized coefficients,
\[
\begin{aligned}
\tilde w_F^2-\tilde w_P^2
&=
\frac{(z_F-\rho z_P)^2-(z_P-\rho z_F)^2}
{(1-\rho^2)^2} \\
&=
\frac{z_F^2-z_P^2}{1-\rho^2}.
\end{aligned}
\]
Because \(|\rho|<1\), the denominator is positive. Thus
\[
\tilde w_F^2>\tilde w_P^2
\iff
z_F^2>z_P^2,
\]
which is equivalent to
\(\lvert\tilde w_F\rvert>\lvert\tilde w_P\rvert
\iff
\lvert z_F\rvert>\lvert z_P\rvert\).
\end{proof}

\proofsection{prop:budgeted-gaussian-mcf}{Budgeted Gaussian MCF}{app:pf-budgeted-gaussian-mcf}
\begin{proof}
From \Cref{prop:ext-3}, we have
\[
g_\tau^*(y)=\mathds{1}\{s(y)\ge \lambda_\tau\},
\qquad
s(y):=\frac{p_1(y)-p_0(y)}{p(y)}.
\]
Write
\[
r(y):=\frac{p_1(y)}{p_0(y)}.
\]
Then
\[
s(y)
=
\frac{p_1(y)-p_0(y)}{\pi p_1(y)+(1-\pi)p_0(y)}
=
\frac{r(y)-1}{\pi r(y)+(1-\pi)}.
\]
The map \(r\mapsto(r-1)/(\pi r+1-\pi)\) is strictly increasing on \((0,\infty)\). Hence thresholding \(s(y)\) is equivalent to thresholding \(r(y)\), or equivalently \(\log r(y)\). In the equal-covariance Gaussian case,
\[
\log r(y)=w^\top y-c,
\]
with \(w=\Sigma^{-1}(\mu_1-\mu_0)\). Therefore, the optimal rule is
\[
g_\tau^*(y)=\mathds{1}\{w^\top y\ge t_\tau\},
\]
for a suitable threshold \(t_\tau\) chosen to satisfy the budget.
\end{proof}

\proofsection{prop:small-budget-selection}{Small budgets rarely select texts that are polite but not formal}{app:pf-thm-budget}
\begin{proof}
Let
\[
D:=P-F.
\]
Since \((F,P)\) is jointly Gaussian and both \(D\) and \(S\) are linear combinations of \((F,P)\), the pair \((D,S)\) is jointly Gaussian.

Its covariance is
\[
\operatorname{Cov}(D,S)
=
\operatorname{Cov}(P-F,\alpha_F F+\alpha_P P).
\]
Using \(\operatorname{Var}(F)=\operatorname{Var}(P)=1\) and \(\operatorname{Cov}(F,P)=\rho\), we obtain
\begin{align*}
\operatorname{Cov}(D,S)
&=
\alpha_F\operatorname{Cov}(P-F,F)+\alpha_P\operatorname{Cov}(P-F,P) \\
&=
\alpha_F(\rho-1)+\alpha_P(1-\rho) \\
&=
(1-\rho)(\alpha_P-\alpha_F).
\end{align*}
Since \(\alpha_F>\alpha_P\) and \(\rho<1\), this is strictly negative:
\[
\operatorname{Cov}(D,S)<0.
\]

Because \((D,S)\) is jointly Gaussian, the conditional distribution of \(D\) given \(S=s\) is Gaussian with mean
\[
\mathbb E[D\mid S=s]
=
\mathbb E[D]
+
\frac{\operatorname{Cov}(D,S)}{\operatorname{Var}(S)}
\bigl(s-\mathbb E[S]\bigr),
\]
and variance
\[
\operatorname{Var}(D\mid S=s)
=
\operatorname{Var}(D)-\frac{\operatorname{Cov}(D,S)^2}{\operatorname{Var}(S)},
\]
which does not depend on \(s\).

Hence
\[
\Pr(E_\kappa\mid S=s)
=
\Pr(D\ge \kappa\mid S=s)
=
1-\Phi\!\left(
\frac{\kappa-\mathbb E[D\mid S=s]}{\sqrt{\operatorname{Var}(D\mid S)}}
\right),
\]
where \(\Phi\) is the standard normal cdf.

Since \(\operatorname{Cov}(D,S)<0\), the conditional mean \(\mathbb E[D\mid S=s]\) is strictly decreasing in \(s\). Therefore, the conditional event probability \(\Pr(E_\kappa\mid S=s)\) is also strictly decreasing in \(s\).

Now
\[
q(t)=\Pr(E_\kappa\mid S\ge t)
=
\mathbb E\!\left[\Pr(E_\kappa\mid S)\mid S\ge t\right].
\]
Because \(\Pr(E_\kappa\mid S=s)\) is decreasing in \(s\), conditioning on larger values of \(S\) can only decrease this expectation. Thus \(q(t)\) is decreasing in \(t\). This proves (1).

For (2), note that
\[
\Pr(E_\kappa\mid t_2\le S<t_1)
=
\mathbb E\!\left[\Pr(E_\kappa\mid S)\mid t_2\le S<t_1\right].
\]
Since \(\Pr(E_\kappa\mid S=s)\) is decreasing in \(s\), its average over the lower interval \([t_3,t_2)\) is at least as large as its average over the higher interval \([t_2,t_1)\). This proves
\[
\Pr(E_\kappa\mid t_2\le S<t_1)
\le
\Pr(E_\kappa\mid t_3\le S<t_2).
\]

Finally, (3) follows because the budget threshold \(t_\tau\) decreases as \(\tau\) increases, whereas part (1) shows that \(\Pr(E_\kappa\mid S\ge t)\) decreases in \(t\).
\end{proof}

\section{What the MIF--MCF objective selects}
\label{sec:hdto-selection}

We provide two algebraic characterizations that illustrate what the MIF--MCF objective prefers to capture when one side of the pair is held fixed.

\textbf{Fixing the outcome feature.}
For a fixed outcome feature \(g\), define
\[
    r_g(a)
    =
    \E\!
    \left[
        g(Y)-\E\{g(Y)\mid X\}
        \mid A=a
    \right].
\]
This is the average residual outcome-feature value among units that received treatment object $a$.
\begin{proposition}[Selection form for fixed outcome feature]
\label{prop:hdto-fixed-g}
Fix \(g\), and suppose \(\epsilon\leq f\leq 1-\epsilon\). Then
\[
    \mathcal{J}(f,g)=\E\{f(A)r_g(A)\}.
\]
An optimizer over all measurable $f$ satisfying $\epsilon\leq f\leq 1-\epsilon$ is
\[
    f_g^*(a)
    =
    \epsilon+(1-2\epsilon)\mathds{1}\{r_g(a)>0\}.
\]
On the tie set~\(\{a:r_g(a)=0\}\), any value in~\([\epsilon,1-\epsilon]\) is optimal.
\end{proposition}
The proof is in \Cref{app:pf-hdto-fixed-g}. Thus, once the outcome feature is fixed, the best treatment feature gives high scores to treatment objects followed by high residual outcome scores. In the documentation example, if detailed AI suggestions are followed by unusually complete notes among comparable visits, then such suggestions receive high values of $f_g^*$.

\textbf{Fixing the treatment feature.}
For a fixed treatment feature \(f\), define
\[
    s_f(y)
    =
    \E\!
    \left[
        f(A)-\E\{f(A)\mid X\}
        \mid Y=y
    \right].
\]
This is the average residual treatment-feature value among units whose outcome object is $y$.
\begin{proposition}[Selection form for fixed treatment feature]
\label{prop:hdto-fixed-f}
Fix \(f\), and suppose \(0\leq g\leq1\). Then
\[
    \mathcal{J}(f,g)=\E\{g(Y)s_f(Y)\}.
\]
An optimizer over all measurable $g$ satisfying $0\leq g\leq1$ is
\[
    g_f^*(y)=\mathds{1}\{s_f(y)>0\}.
\]
\end{proposition}
The proof is in \Cref{app:pf-hdto-fixed-f}. Together, the two selection forms show that the criterion searches for a mutually reinforcing pair: a treatment feature that selects treatment objects linked to high residual outcomes, and an outcome feature that selects outcome objects linked to high residual treatments.

\proofsection{prop:hdto-fixed-g}{Selection form for fixed outcome feature}{app:pf-hdto-fixed-g}
\begin{proof}
By iterated expectation,
\[
\begin{aligned}
\mathcal{J}(f,g)
&=
\E\!
\left[
    f(A)\{g(Y)-\E(g(Y)\mid X)\}
\right]
\\
&=
\E\!
\left[
    f(A)
    \E\!\left\{g(Y)-\E(g(Y)\mid X)\mid A\right\}
\right]
=
\E\{f(A)r_g(A)\}.
\end{aligned}
\]
The last expression is linear in \(f(a)\) at each value of \(a\). It is maximized by assigning \(1-\epsilon\) when \(r_g(a)>0\), assigning \(\epsilon\) when \(r_g(a)<0\), and assigning any value in \([\epsilon,1-\epsilon]\) when \(r_g(a)=0\). The displayed optimizer uses \(\epsilon\) on the tie set.
\end{proof}

\proofsection{prop:hdto-fixed-f}{Selection form for fixed treatment feature}{app:pf-hdto-fixed-f}
\begin{proof}
By \Cref{prop:hdto-equivalent},
\[
    \mathcal{J}(f,g)
    =
    \E\!
    \left[
        g(Y)\{f(A)-\E(f(A)\mid X)\}
    \right].
\]
Conditioning on \(Y\) gives
\[
\begin{aligned}
\mathcal{J}(f,g)
&=
\E\!
\left[
    g(Y)
    \E\!\left\{f(A)-\E(f(A)\mid X)\mid Y\right\}
\right]
=
\E\{g(Y)s_f(Y)\}.
\end{aligned}
\]
The last expression is linear in \(g(y)\) at each value of \(y\). It is maximized by setting \(g(y)=1\) when \(s_f(y)>0\) and \(g(y)=0\) when \(s_f(y)<0\).
\end{proof}

\section{Covariate-adaptive treatment and outcome features}
\label{sec:hdto-heterogeneous}

The global version learns one treatment feature and one outcome feature for the full population. Some scientific questions call for features whose meaning varies across covariates. In clinical documentation, a useful suggestion for an emergency visit may differ from a useful suggestion for a routine follow-up, and the corresponding change in the final note may differ as well.

We allow this by learning
\[
    f=f(A,X),
    \qquad
    g=g(Y,X).
\]
The same suggestion or final note may then receive different scores at different covariate values. The covariate-adaptive objective is
\[
    \mathcal{J}_X(f,g)
    =
    \E\!
    \left[
        f(A,X)\{g(Y,X)-\E(g(Y,X)\mid X)\}
    \right].
\]
By the same calculation as in \Cref{prop:hdto-equivalent},
\[
    \mathcal{J}_X(f,g)
    =
    \E\!
    \left[
        \operatorname{Cov}\{f(A,X),g(Y,X)\mid X\}
    \right].
\]
The interpretation is local to the covariate value: the learned pair captures treatment and outcome directions whose association remains within each covariate stratum.

The negative-outcome identity also carries over. If
\[
    Y^\dagger\mid X \sim Y\mid X,
    \qquad
    Y^\dagger \perp A\mid X,
\]
then
\[
    \E\{f(A,X)g(Y,X)\}
    -
    \E\{f(A,X)g(Y^\dagger,X)\}
    =
    \mathcal{J}_X(f,g).
\]
In finite samples, the matched baseline evaluates each candidate negative outcome at the target covariate value,
\[
    \widehat b_\beta(X_i)
    =
    \sum_{j\in N(i)} w_{ij}g_\beta(Y_j,X_i).
\]
The empirical objective is
\[
    \widehat{\mathcal{J}}_X(\gamma,\beta)
    =
    \frac{1}{n}
    \sum_{i=1}^n
    f_\gamma(A_i,X_i)
    \left\{
        g_\beta(Y_i,X_i)
        -
        \sum_{j\in N(i)}w_{ij}g_\beta(Y_j,X_i)
    \right\}.
\]
Removing $X_i$ from the arguments of $f_\gamma$ and $g_\beta$ recovers the global estimator. The covariate-adaptive version is useful when the scientific question concerns subgroup-specific treatment and outcome directions. The global version is the clearer starting point when the goal is one interpretable feature of the treatment and one interpretable feature of the outcome.

\section{Additional empirical details and results}
\label{app:empirical}

This appendix retains the complete experimental mechanisms, auxiliary results, and additional examples referenced from the main empirical studies. The studies appear in the same within-experiment order as in the main text.

\subsection{Experiment I: Additional details for text-based outcomes}
\label{app:emp-text-outcomes}

\subsubsection{GYAFC formality}
\label{app:emp-formality}

\textbf{Experimental setup.}
The binary context~$X$ equals zero for entertainment and one for family. The binary treatment~$A$ equals one under exposure to the formality-inducing intervention and zero under the baseline condition. The binary style label~$F$ equals one for formal sentences and zero for informal sentences. The observed outcome~$Y$ is a sentence sampled from the corresponding context and style category. Treatment assignment follows
\[
X\sim\operatorname{Bernoulli}(0.5),\qquad A\mid X\sim\operatorname{Bernoulli}\{\operatorname{sigmoid}(0.35+0.3X)\}.
\]
Thus, treatment is not randomized marginally with respect to context. Because context also affects the outcome-text distribution, the context~$X$ is a confounder.

\textbf{Treatment scenarios.}
In Scenario~1, the outcome-style mechanism is
\[
F\mid A,X\sim\operatorname{Bernoulli}\{\operatorname{sigmoid}(-0.5+1.5A+0.2X)\}.
\]
Treatment therefore increases formality in both contexts. The learned feature-scoring function should recover a global formality direction: formal sentences should receive high values of~$\hat g(Y)$, while informal sentences should receive low values. In Scenario~2, the outcome-style mechanism is
\[
F\mid A,X\sim\operatorname{Bernoulli}\{\operatorname{sigmoid}(1.5(2X-1)A+0.1X)\}.
\]
Treatment decreases formality when~$X=0$ and increases formality when~$X=1$. This mechanism represents a context-dependent writing intervention: it may encourage a more conversational style in entertainment contexts but a more formal style in family or advice-related contexts. A valid feature-scoring function should therefore recover the treatment-associated textual direction relative to the context-specific baseline, rather than merely a global formality score.

\textbf{Estimation.}
We estimate~$\hat g(Y)$ in Scenario~1 and~$\hat g(Y,X)$ in Scenario~2. Text outcomes are embedded before training, and all reported summaries use a held-out validation set.

\textbf{Additional score examples.}
\Cref{tab:formality-s1-examples,tab:formality-s2-examples} provide the complete low- and high-score examples. In Scenario~1, high-scoring outcomes are more formal in both contexts, while low-scoring outcomes are more informal or conversational. In Scenario~2, the interpretation of high values depends on context: high-scoring outcomes are more informal in entertainment contexts and more formal in family contexts. These qualitative examples support the interpretation of the learned feature-scoring function as the outcome-side direction induced by treatment.

\begin{table}[htbp]
    \centering
    \small
    \begin{tabularx}{\textwidth}{>{\raggedright\arraybackslash}p{0.19\textwidth}XX}
        \toprule
        Context & Low $\hat g(Y)$ examples & High $\hat g(Y)$ examples \\
        \midrule
        Entertainment
        &
        \begin{minipage}[t]{\linewidth}\raggedright
            \textit{and even if you think it is...i will always support you BABy..}\\
            \textit{no she was never a mouseketeer!}\\
            \textit{best-nothing but a good time--poison worst-i saw red--warrant}
        \end{minipage}
        &
        \begin{minipage}[t]{\linewidth}\raggedright
            \textit{I would prefer not to share my fame with a runner-up if I won.}\\
            \textit{The blonde asked, pointing at one of the sex toys on the shelf.}\\
            \textit{I design games using PowerPoint when I am able to do so. My friend greatly enjoys producing PowerPoint games.}
        \end{minipage}
        \\
        \midrule
        Family
        &
        \begin{minipage}[t]{\linewidth}\raggedright
            \textit{it depends... if u r losing it..u will b hurt.. (life is a love...love it)  }\\
            \textit{by usher wine by r kelly and thing that is soft r\&b that has a good beat}\\
            \textit{it doesn't matter that she's black, a girl is a girl!!}
        \end{minipage}
        &
        \begin{minipage}[t]{\linewidth}\raggedright
            \textit{Based off of your age of fifty-five and his age of sixty, you are both old enough.  }\\
            \textit{Be gracious and mature about it.}\\
            \textit{It is important to look at other important qualities as well.}
        \end{minipage}
        \\
        \bottomrule
    \end{tabularx}
    \caption{The feature-scoring function has the same qualitative formality interpretation in both contexts. Representative text outcomes with low and high learned values in Scenario 1. High $\hat g(Y)$ examples are expected to be more formal in both contexts.}
    \label{tab:formality-s1-examples}
\end{table}

\begin{table}[htbp]
    \centering
    \small
    \begin{tabularx}{\textwidth}{>{\raggedright\arraybackslash}p{0.19\textwidth}XX}
        \toprule
        Context & Low $\hat g(Y,X)$ examples & High $\hat g(Y,X)$ examples \\
        \midrule
        Entertainment
        &
        \begin{minipage}[t]{\linewidth}\raggedright
            \textit{This is a song that you enjoy.  }\\
            \textit{I am the most sexually attrictive caucasian boy to ever exist.}\\
            \textit{I recommend listing to "Anything Goes" featuring Capone N. Noreaga if you do not believe me}
        \end{minipage}
        &
        \begin{minipage}[t]{\linewidth}\raggedright
            \textit{anywayz fred astaire was big...you can try to search at yahoo music..}\\
            \textit{no they don't thay all ugly}\\
            \textit{I think paris should have won or that bald dude!!}
        \end{minipage}
        \\
        \midrule
        Family
        &
        \begin{minipage}[t]{\linewidth}\raggedright
            \textit{and at night its a long ass conversation mostly dead aur.}\\
            \textit{go up to her and be all do you like me, cuz i like you?}\\
            \textit{She is avoiding u, it means she may be in love with some other guy.}
        \end{minipage}
        &
        \begin{minipage}[t]{\linewidth}\raggedright
            \textit{My boyfriend is five years older than I am.}\\
            \textit{Naturally, it always helps to have someone nice to start dating.}\\
            \textit{They have insufficient funds to purchase a larger bra.}
        \end{minipage}
        \\
        \bottomrule
    \end{tabularx}
    \caption{A context-free interpretation of the learned score would be misleading. Representative text outcomes with low and high learned values in Scenario 2. High $\hat g(Y,X)$ examples are expected to be more informal in the entertainment context and more formal in the family context.}
    \label{tab:formality-s2-examples}
\end{table}

\textbf{Complete nudging trajectories.}
We further evaluate the learned direction using embedding-based nudging~\citep{wibisono2026causal}. Starting from an initial text outcome, nudging moves its embedding in the direction that increases the learned feature score while holding context fixed. At each iteration, the updated embedding is decoded back into text and evaluated with an external formality scorer. This procedure checks whether higher feature scores produce the transformation predicted by the data-generating process.

\Cref{fig:formal-nudging} preserves every plotted iteration. In Scenario~1, nudging increases average formality monotonically from approximately $0.15$ at the initial texts to $0.82$ after ten iterations. This pattern agrees with the data-generating process, in which treatment increases formality in both contexts.

In Scenario~2, the nudging direction depends on context. For initially informal family text, corresponding to $(F,X)=(0,1)$, average formality increases from approximately $0.28$ to $0.63$. For initially formal entertainment text, corresponding to $(F,X)=(1,0)$, average formality instead decreases from approximately $0.84$ to $0.72$. Thus, the learned transformation is not simply a global push toward formal language; it moves text in the direction associated with treatment within each context.

\begin{figure}[htbp]
\centering
\begin{subfigure}[b]{0.43\textwidth}
\centering
\includegraphics[width=\textwidth]{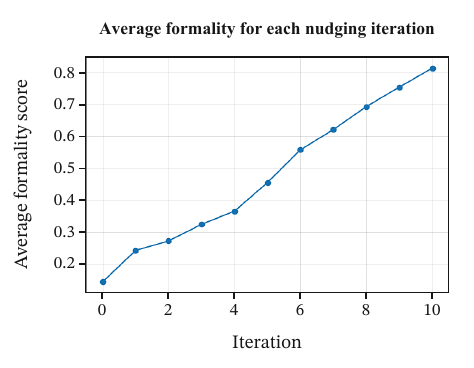}
\caption{Scenario 1}
\end{subfigure}
\hfill
\begin{subfigure}[b]{0.47\textwidth}
\centering
\includegraphics[width=\textwidth]{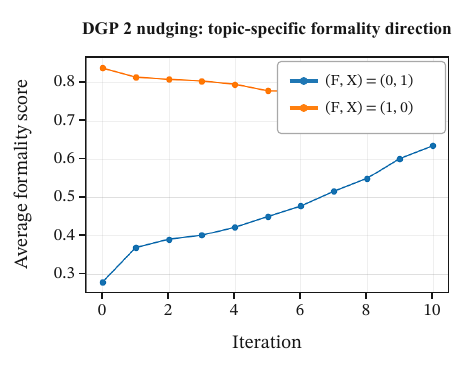}
\caption{Scenario 2}
\end{subfigure}
\caption{A single global editing direction cannot describe both formality scenarios. Average formality scores over nudging iterations in the formality experiment. In Scenario 1, nudging increases formality because the treatment induces more formal text in both contexts. In Scenario 2, the nudging direction is context-dependent: it increases formality in the family context but decreases formality in the entertainment context.}
\label{fig:formal-nudging}
\end{figure}

\FloatBarrier
\subsubsection{GYAFC topic-model comparison}
\label{app:emp-formality-lda}

\textbf{Experimental setup.}
The topic-model comparison instantiates the codebook framework of \citet{egami2022make} with latent Dirichlet allocation~\citep{blei2003latent}, focusing on Scenario~1, the context-invariant case. We fit models for topic counts~$K\in\{2,3,\ldots,10\}$. The dominant-topic function~$D(Y)\in[K]$ assigns each document to the topic with the largest posterior proportion. For each fitted model, we identify the topic subset~$\mathcal T\subset[K]$ that maximizes
\[
\mathbb E\!\left[\mathds{1}\{D(Y(1))\in\mathcal T\}\right]-\mathbb E\!\left[\mathds{1}\{D(Y(0))\in\mathcal T\}\right].
\]
We compare the topics in~$\mathcal T$ with those outside~$\mathcal T$ to infer which aspects of the text are most causally affected by treatment. The model with~$K=8$ produces the largest contrast. \Cref{tab:lda-sum} retains all representative words.

\begin{table}[H]
\small
\caption{In the GYAFC study, the MCF yields a more directly interpretable feature-scoring function than the LDA summary. Representative words for topics in $\mathcal{T}$ (positively affected by the treatment) and topics not in $\mathcal{T}$ from the LDA model with $K = 8$. It is difficult to interpret exactly what aspects of the text $\mathcal{T}$ corresponds to, reflecting a known limitation of topic model-based approaches.}
\centering
\begin{tabular}{c @{\hspace{10mm}} c}
\toprule
Topics in $\mathcal{T}$ & Topics not in $\mathcal{T}$ \\
\midrule
\makecell{\textit{yes}, \textit{com}, \textit{think}, \textit{good}, \textit{try}, \\
          \textit{yahoo}, \textit{like}, \textit{search}, \textit{girl}, \textit{music}}
& \makecell{\textit{don't}, \textit{way}, \textit{just}, \textit{love}, \textit{know}, \\
            \textit{mean}, \textit{say}, \textit{idea}, \textit{makes}, \textit{problem}} \\
\midrule
\makecell{\textit{person}, \textit{ask}, \textit{care}, \textit{thing}, \textit{look}, \\
          \textit{free}, \textit{work}, \textit{self}, \textit{does}, \textit{life}}
& \makecell{\textit{song}, \textit{like}, \textit{man}, \textit{did}, \textit{just}, \\
            \textit{really}, \textit{love}, \textit{think}, \textit{best}, \textit{say}} \\
\midrule
\makecell{\textit{tell}, \textit{need}, \textit{old}, \textit{like}, \textit{make}, \\
          \textit{big}, \textit{talk}, \textit{just}, \textit{want}, \textit{good}}
& \makecell{\textit{know}, \textit{time}, \textit{don't}, \textit{men}, \textit{think}, \\
            \textit{like}, \textit{women}, \textit{right}, \textit{just}, \textit{better}} \\
\midrule
\makecell{\textit{movie}, \textit{like}, \textit{heard}, \textit{great}, \textit{funny}, \\
          \textit{just}, \textit{attractive}, \textit{that's}, \textit{listen}, \textit{good}}
& \makecell{\textit{love}, \textit{like}, \textit{question}, \textit{relationship}, \textit{sure}, \\
            \textit{want}, \textit{answer}, \textit{man}, \textit{hope}, \textit{friend}} \\
\bottomrule
\end{tabular}
\label{tab:lda-sum}
\end{table}

\textbf{Results.}
The fitted LDA topics summarize clusters of co-occurring words, but in this GYAFC study the selected and unselected topic summaries do not map directly to formality, sentiment, or another clear stylistic dimension. This makes it difficult to determine which text properties treatment affects. The MCF instead yields a feature-scoring function that maps more directly to the formality dimension in the data-generating process and results.

\FloatBarrier
\subsubsection{ParaDetox toxicity}
\label{app:emp-toxicity}

\textbf{Experimental setup.}
The binary context~$X$ equals zero for a high-conflict discussion and one for a moderated support setting. The binary treatment~$A$ equals one under exposure to the platform intervention and zero under the baseline condition. The binary label~$T$ equals one for a non-toxic or neutral sentence and zero for a toxic sentence. The observed outcome~$Y$ is sampled from the corresponding context and toxicity category. Treatment assignment follows
\[
X\sim\operatorname{Bernoulli}(0.5),\qquad A\mid X\sim\operatorname{Bernoulli}\{\operatorname{sigmoid}(0.35+0.3X)\}.
\]
The context~$X$ is a confounder because it affects both treatment assignment and the outcome-text distribution.

\textbf{Treatment scenarios.}
In Scenario~1, the toxicity mechanism is
\[
T\mid A,X\sim\operatorname{Bernoulli}\{\operatorname{sigmoid}(-0.5+1.5A+0.3X)\}.
\]
Treatment increases non-toxicity in both contexts. The learned feature-scoring function should therefore recover a global safety direction: it should assign high values of~$\hat g(Y)$ to non-toxic sentences and low values to toxic sentences. In Scenario~2, the mechanism is
\[
T\mid A,X\sim\operatorname{Bernoulli}\{\operatorname{sigmoid}(1.5(2X-1)A+0.1X)\}.
\]
Treatment decreases non-toxicity in high-conflict discussions and increases non-toxicity in moderated support settings. This mechanism represents a context-dependent moderation or reply-assistance intervention: it may amplify confrontational language in one setting while encouraging safer or more neutral language in the other. The target is therefore context-specific. A valid feature-scoring function should recover the treatment-associated direction relative to the context-specific baseline rather than simply a global toxicity classifier.

\textbf{Estimation.}
We estimate~$\hat g(Y)$ in Scenario~1 and~$\hat g(Y,X)$ in Scenario~2. Text outcomes are embedded before training, and all reported summaries use a held-out validation set.

\textbf{Additional score examples.}
\Cref{tab:toxicity-s1-examples,tab:toxicity-s2-examples} retain every low- and high-score example. In Scenario~1, high-scoring outcomes are more non-toxic or neutral in both contexts, while low-scoring outcomes are more toxic. In Scenario~2, the interpretation of high values depends on context: high values correspond to more toxic text in high-conflict discussions and to more non-toxic text in moderated support settings. These qualitative examples support the interpretation of the learned feature-scoring function as the context-specific safety direction induced by treatment.

\begin{table}[htbp]
    \centering
    \small
    \begin{tabularx}{\textwidth}{>{\raggedright\arraybackslash}p{0.19\textwidth}XX}
        \toprule
        Context & Low $\hat g(Y)$ examples & High $\hat g(Y)$ examples \\
        \midrule
        \makecell[l]{High-conflict\\discussion}
        &
        \begin{minipage}[t]{\linewidth}\raggedright
            \textit{im in a school you bitch fuck .}\\
            \textit{they better fucking nerf shoes .}\\
            \textit{you lose in our wars every fucking time !}
        \end{minipage}
        &
        \begin{minipage}[t]{\linewidth}\raggedright
            \textit{Didn't I just call you?	}\\
            \textit{Who makes decisions like that?}\\
            \textit{What is this war?}
        \end{minipage}
        \\
        \midrule
        \makecell[l]{Moderated support\\setting}
        &
        \begin{minipage}[t]{\linewidth}\raggedright
            \textit{they analyzed their own fucking network logs of their own computer systems .}\\
            \textit{it worked but it 's talkin bout some other shit i ain 't fuckin wit it so fuck it}\\
            \textit{there are crazy fucking assholes everywhere}
        \end{minipage}
        &
        \begin{minipage}[t]{\linewidth}\raggedright
            \textit{I fear nothing}\\
            \textit{Everybody knew that, but me -}\\
            \textit{It's too early for this on my timeline.}
        \end{minipage}
        \\
        \bottomrule
    \end{tabularx}
    \caption{The feature-scoring function has the same qualitative safety interpretation in both contexts. Representative text outcomes with low and high learned values in Scenario 1. High $\hat g(Y)$ examples are expected to be more non-toxic in both contexts.}
    \label{tab:toxicity-s1-examples}
\end{table}

\begin{table}[htbp]
    \centering
    \small
    \begin{tabularx}{\textwidth}{>{\raggedright\arraybackslash}p{0.19\textwidth}XX}
        \toprule
        Context & Low $\hat g(Y,X)$ examples & High $\hat g(Y,X)$ examples \\
        \midrule
        \makecell[l]{High-conflict\\discussion}
        &
        \begin{minipage}[t]{\linewidth}\raggedright
            \textit{He wasnt playin, he remembers it}\\
            \textit{Do I know you?}\\
            \textit{I don't like zoo}
        \end{minipage}
        &
        \begin{minipage}[t]{\linewidth}\raggedright
            \textit{nutella and ferrero rocher are fucking amazing .	}\\
            \textit{oh shit the video fuck , haven 't even started ! ! !	}\\
            \textit{how the fuck are they even gonna get to mexico , honestly ?}
        \end{minipage}
        \\
        \midrule
        \makecell[l]{Moderated support\\setting}
        &
        \begin{minipage}[t]{\linewidth}\raggedright
            \textit{dress how the fuck you want .	}\\
            \textit{why do we continue to accept to be fucked like this .}\\
            \textit{looks pretty fucking accurate to me .	}
        \end{minipage}
        &
        \begin{minipage}[t]{\linewidth}\raggedright
            \textit{The people elected uncapable people and then reelect them.}\\
            \textit{Well, this is rusting, it's old but it's still in good shape.}\\
            \textit{Expect too leave the session and announce "you're not good at pilates".}
        \end{minipage}
        \\
        \bottomrule
    \end{tabularx}
    \caption{A context-free safety interpretation of the learned score would be misleading. Representative text outcomes with low and high learned values in Scenario 2. High $\hat g(Y,X)$ examples are expected to be more toxic in the high-conflict context and more non-toxic in the moderated support setting.}
    \label{tab:toxicity-s2-examples}
\end{table}

\textbf{Complete nudging trajectories.}
We use the same embedding-based nudging procedure as in the formality study. At each iteration, the text representation moves in the direction that increases the learned feature score, the updated representation is decoded back into text, and the decoded text is evaluated using non-toxicity, defined as one minus Detoxify's toxicity score~\citep{Detoxify}.

\Cref{fig:toxic-nudging} preserves every plotted iteration. In Scenario~1, nudging increases average non-toxicity monotonically from approximately $0.04$ at the initial texts to $0.77$ after ten iterations. This pattern agrees with the data-generating process, in which treatment increases non-toxicity in both contexts.

In Scenario~2, the nudging direction depends on context. For initially toxic text in moderated support settings, corresponding to $(T,X)=(0,1)$, average non-toxicity increases from approximately $0.05$ to $0.72$. For initially non-toxic text in high-conflict discussions, corresponding to $(T,X)=(1,0)$, average non-toxicity instead decreases from approximately $0.96$ to $0.83$. Thus, the learned transformation is not simply a global push toward safer language; it follows the context-dependent direction induced by treatment.

\begin{figure}[H]
\centering
\begin{subfigure}[b]{0.39\textwidth}
\centering
\includegraphics[width=\textwidth]{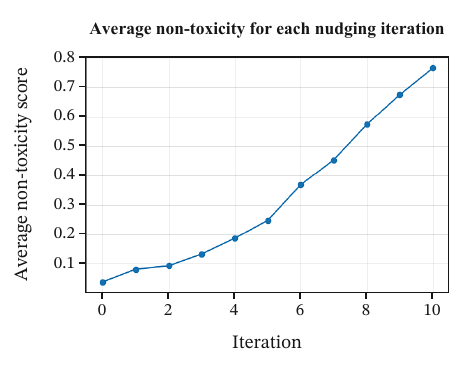}
\caption{Scenario 1}
\end{subfigure}
\hfill
\begin{subfigure}[b]{0.43\textwidth}
\centering
\includegraphics[width=\textwidth]{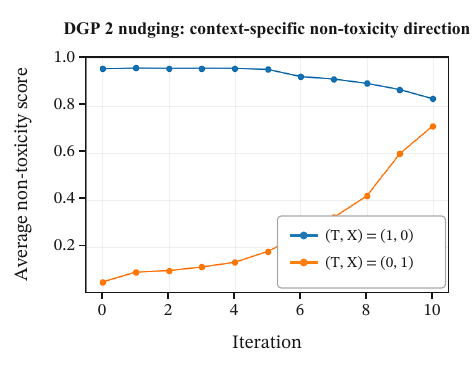}
\caption{Scenario 2}
\end{subfigure}
\caption{A single global safety-editing direction cannot describe both scenarios. Average non-toxicity scores over nudging iterations in the toxicity experiment. In Scenario 1, nudging increases non-toxicity because the intervention induces safer text in both contexts. In Scenario 2, the nudging direction is context-dependent: it increases non-toxicity in the moderated support setting but decreases non-toxicity in the high-conflict discussion setting.}
\label{fig:toxic-nudging}
\end{figure}

\FloatBarrier
\subsubsection{Multiple treatment-induced textual dimensions}
\label{app:emp-multiple-dimensions}

\textbf{Experimental setup.}
The study retains the GYAFC treatment assignment but changes the outcome mechanism. The formality label~$F$ equals one for formal text and zero for informal text. The punctuation label~$P$ equals one if the text contains a question mark or exclamation mark and zero otherwise. The observed outcome~$Y$ is sampled from the corresponding text category. Conditional on treatment, the two textual dimensions~$F$ and~$P$ are independent.

Under the control condition~$A=0$, both attributes have probability~$0.5$. The three scenarios are as follows.
\begin{enumerate}[leftmargin=*]
\item In Scenario~1, treatment affects both dimensions: when~$A=1$, both formality and punctuation have probability~$0.8$.
\item In Scenario~2, treatment affects only punctuation: when~$A=1$, formality remains at probability~$0.5$, while punctuation increases to probability~$0.8$.
\item In Scenario~3, treatment affects only formality: when~$A=1$, formality increases to probability~$0.8$, while punctuation remains at probability~$0.5$.
\end{enumerate}
Thus, the treatment-induced outcome feature differs across scenarios. Scenario~1 requires a joint formality-and-punctuation feature, Scenario~2 requires punctuation but not formality, and Scenario~3 requires formality but not punctuation. We represent text using the StyleDistance algorithm of \citet{patel2024styledistance} and estimate a homogeneous MCF in each scenario. \Cref{tab:mult-dim} in the main text reports the average learned score for every combination of formality and punctuation.

\FloatBarrier
\subsection{Experiment II: Additional details for image-based outcomes}
\label{app:emp-image-outcomes}

\subsubsection{Representation and data-generating process}
\label{app:emp-image-representation}

\textbf{Experimental setup.}
The study includes BBBC005v1 synthetic cell-body-stain images~\citep{ljosa2012annotated} with at most~$40$ cells. Content is the cell count, and style comprises the remaining image attributes, with focus blur as the target attribute. We use an adversarial autoencoder,\footnote{\url{https://github.com/johncf/mnist-style}} with architectural adjustments for the cell images and a $32$-dimensional style representation.

The binary covariate~$X$ indicates that cell count is below~$20$. The treatment and target-blur mechanisms are
\[
X\sim\operatorname{Bernoulli}(0.5),\qquad A\mid X\sim\operatorname{Bernoulli}\{\operatorname{sigmoid}(0.35+0.3X)\},
\]
and
\[
B^\star=8+24A+4X+\varepsilon,\qquad \varepsilon\sim\mathcal N(0,5^2).
\]
Here, $\varepsilon$ is the Gaussian noise term. The target blur~$B^\star$ is truncated to the empirical blur range. Within the cell-count group matching~$X$, a raw image~$O$ is sampled uniformly from the $30$ images with blur closest to~$B^\star$. The represented outcome~$Y=\psi(O)$ is the selected image's style representation, and the observed blur~$F$ is used only to interpret the learned score.

\subsubsection{Additional image examples}
\label{app:emp-image-examples}

\Cref{fig:img_example_cell} shows style transfer for cell counts $1$, $5$, $10$, $14$, $18$, $23$, $27$, $31$, $35$, and~$40$. In each row, every cell image follows the style of the leftmost image. The number of cells systematically increases from left to right, while the level of blur remains approximately fixed. This pattern supports the content--style separation used in the outcome analysis.

\begin{figure}[htbp]
    \centering
    \includegraphics[width=\textwidth]{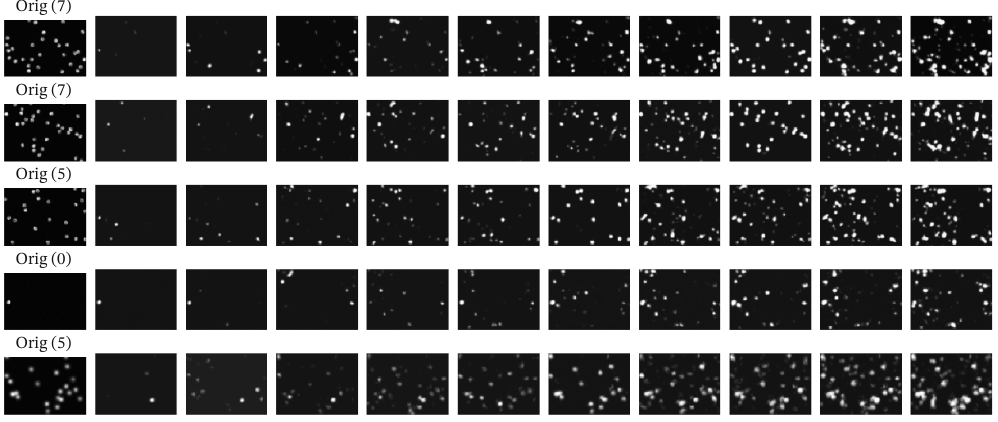}
    \caption{Style transfer preserves blur while cell-count content changes. Image style transfer example. Each row consists of how each number of cells (from 1 to 40) looks like following the style of the leftmost image.}
\label{fig:img_example_cell}
\end{figure}

\FloatBarrier
\subsection[Experiment III: Additional details for simultaneous text-based treatments and outcomes]{Experiment III: Additional details for simultaneous text-based\protect\\ treatments and outcomes}
\label{app:emp-text-treatment-outcomes}

\subsubsection{Stylized treatment-plan example}
\label{app:emp-toy}

\textbf{Experimental setup.}
The binary covariate follows~\(X\sim\operatorname{Bernoulli}(0.5)\). The treatment vector is~\(A=(A_1,A_2,A_3)\). Its three coordinates are independent~\(\operatorname{Uniform}(0,1)\) variables. The outcome vector is~\(Y=(Y_1,Y_2,Y_3)\), where
\[
Y_1\mid A_1\sim\mathcal N(A_1,0.2^2),\qquad Y_2,Y_3\sim\operatorname{Uniform}(0,1).
\]
Only~$A_1$ affects the outcome, and larger values of~$A_1$ produce larger values of~$Y_1$. The homogeneous MIF--MCF estimator learns~$\hat f(A)$ and~$\hat g(Y)$; the other coordinates are deliberately irrelevant.

\subsubsection{Formality-inducing prompts in news headline generation}
\label{app:emp-headline}

\textbf{Experimental setup.}
The complete generation mechanism has four steps.
\begin{enumerate}[leftmargin=*]
\item The binary topic follows~$X\sim\operatorname{Bernoulli}(0.5)$, where zero denotes entertainment and one denotes politics.
\item The prompt-formality indicator follows~$F\sim\operatorname{Bernoulli}\{\operatorname{sigmoid}(2X-1)\}$, where one denotes a formal prompt.
\item The treatment text~$A$ is sampled from a prompt family indexed by~$F$. A formal example is ``Compose a concise news headline in a professional tone.'' An informal example is ``Write a compact headline with a relaxed tone.'' The topic is not included in~$A$.
\item Qwen2.5-1.5B-Instruct~\citep{qwen2024qwen25} generates the outcome~$Y$ from the concatenated instruction consisting of~$A$ followed by~$T(X)$, where~$T(0)$ is ``The headline must be about entertainment'' and~$T(1)$ is ``The headline must be about politics.''
\end{enumerate}
The topic affects both the probability of receiving a formal prompt and the content of the generated headline, creating the adjustment problem. An algorithm that ignores~$X$ may therefore confuse topic-related variation with prompt-induced style. The desired feature-scoring functions should recover prompt-side formality and outcome-side headline formality while adjusting for topic. Prompt and headline texts are represented with the \textit{all-MiniLM-L6-v2} Sentence-Transformers model~\citep{reimers2019sentence}. Homogeneous functions~$\hat f(A)$ and~$\hat g(Y)$ are estimated with the MIF--MCF objective.

\FloatBarrier

\end{document}